%% file: ICML_main.tex
\documentclass{article}

\usepackage{url}

\usepackage{mathrsfs}
\usepackage{amsthm}
\usepackage{graphicx}
\usepackage{xcolor}
\usepackage{verbatim}
\usepackage{bbm}
\usepackage{amsmath}
\usepackage{amsfonts}
\usepackage{amssymb}
\usepackage{nicefrac}
\usepackage{mathtools}
\usepackage{enumitem}

\usepackage{cancel}

\usepackage{tikz}
\usetikzlibrary{positioning}

\usepackage{caption}
\usepackage{subcaption}

\usepackage{hyperref}
\usepackage[accepted]{icml2025}
\usepackage[capitalise]{cleveref}

\input{pier_macros}
\input{tikz_stuff}
\input{blake_macros}

\newtheorem{theorem}{Theorem}
\newtheorem{lemma}{Lemma}
\newtheorem{proposition}{Proposition}

\theoremstyle{definition}
\newtheorem{definition}{Definition}

\icmltitlerunning{Gradient Descent Converges Linearly to Flatter Minima than Gradient Flow}

\begin{document}

% \maketitle

\twocolumn[
\icmltitle{
Gradient Descent Converges Linearly to Flatter Minima \\ than Gradient Flow in Shallow Linear Networks
% SGD Typically (but Never) Occurs at the Edge of Stability and Explaining Catapults
% Revisiting the Edge of Stability for SGD 
}

% It is OKAY to include author information, even for blind
% submissions: the style file will automatically remove it for you
% unless you've provided the [accepted] option to the icml2025
% package.

% List of affiliations: The first argument should be a (short)
% identifier you will use later to specify author affiliations
% Academic affiliations should list Department, University, City, Region, Country
% Industry affiliations should list Company, City, Region, Country

% You can specify symbols, otherwise they are numbered in order.
% Ideally, you should not use this facility. Affiliations will be numbered
% in order of appearance and this is the preferred way.
\icmlsetsymbol{equal}{*}

\begin{icmlauthorlist}
\icmlauthor{Pierfrancesco Beneventano}{yyy}
\icmlauthor{Blake Woodworth}{xxx}
\end{icmlauthorlist}

\icmlaffiliation{yyy}{Princeton University, Princeton, NJ, USA}
\icmlaffiliation{xxx}{George Washington University, Washington, DC, USA}
% \icmlaffiliation{comp}{Company Name, Location, Country}
% \icmlaffiliation{sch}{School of ZZZ, Institute of WWW, Location, Country}

\icmlcorrespondingauthor{Pierfrancesco Beneventano}{pierb@princeton.edu}

% You may provide any keywords that you
% find helpful for describing your paper; these are used to populate
% the "keywords" metadata in the PDF but will not be shown in the document
\icmlkeywords{Gradient Descent, Optimization for deep learning, Implicit Regularization, Linear Networks}

\vskip 0.3in

]
% this must go after the closing bracket ] following \twocolumn[ ...

% This command actually creates the footnote in the first column
% listing the affiliations and the copyright notice.
% The command takes one argument, which is text to display at the start of the footnote.
% The \icmlEqualContribution command is standard text for equal contribution.
% Remove it (just {}) if you do not need this facility.

\printAffiliationsAndNotice{}  % leave blank if no need to mention equal contribution
% \printAffiliationsAndNotice{\icmlEqualContribution} % otherwise use the standard text.
% \printAffiliations

\begin{abstract}
% \red{Pier proposal:}
We study the gradient descent (GD) dynamics of a depth-2 linear neural network with a single input and output. We show that GD converges at an explicit linear rate to a global minimum of the training loss, even with a large stepsize--about $2/\textrm{sharpness}$. It still converges for even larger stepsizes, but may do so very slowly. We also characterize the solution to which GD converges, which has lower norm and sharpness than the gradient flow solution.
% Moreover, the norm and the sharpness decrease as the step size is increased.
Our analysis reveals a trade off between the speed of convergence and the magnitude of implicit regularization.
% Precisely, the slower the convergence the flatter will be the solution and the higher the implied regularization, the slower the convergence.
This sheds light on the benefits of training at the ``Edge of Stability'', which induces additional regularization by delaying convergence and may have implications for training more complex models.
% \red{Blake first version:}
% We study the gradient descent (GD) dynamics of a depth-2 linear neural net with a single input and output. We show that with a big stepsize, GD converges at an explicit linear rate to a global minimum of the training loss. For larger stepsizes, convergence is still assured, but may be very slow. We also prove that GD implicitly regularizes the sharpness of the solution it reaches, and it reaches a strictly flatter minimum than gradient flow. Our analysis reveals a trade-off between the speed of convergence and the magnitude of regularization which is related to the Edge of Stability phenomenon and has potential implications for training more complex models.
\end{abstract}

\section{Introduction}

Training modern machine learning (ML) models like deep neural networks via empirical risk minimization (ERM) requires solving difficult high-dimensional, non-convex, under-determined optimization problems. Although they are usually intractable to solve in theory, we train models effectively in practice using algorithms like stochastic gradient descent (SGD). This highlights a disconnect between the worst-case convergence rate of SGD and its convergence on specific ERM problems that arise from training, e.g., neural networks. Even if we can solve the ERM problem, typical minimizers of the under-determined objective will overfit and generalize poorly. That said, the specific solutions found by SGD and its variants usually \emph{do} successfully generalize. Understanding how and why we are able to successfully optimize and generalize with these models is of great interest to the ML community and could help fuel continued progress in applied ML. 

A key feature of popular ML models, including neural networks, is that the model output is related to the product of model parameters in successive layers. For instance, the output of a 2 layer feed-forward network with ReLU activations has output $\mathbf{W}_2 \relu(\mathbf{W}_1 x+\b_1)$, which is closely related to the product of the weight matrices $\mathbf{W}_2\mathbf{W}_1$. Ultimately, this ``self-multiplication'' of different model parameters gives rise to the non-convex and under-determined ERM problems that cause such (theoretical) difficulties. 

In this work, we distill this parameter self-multiplication property down to its simplest form and comprehensively explain how it affects the training optimization dynamics, the ``implicit regularization'' of the model parameters, and the ``edge-of-stability'' dynamics that arise in certain regimes. In particular, we consider the extremely simple problem of learning a univariate linear model $\hat{y} = mx$ to minimize the squared error, except we parameterize the slope as $m = m(\a,\b) = \a^\top \b$ in terms of self-multiplying parameters $\a,\b\in\R^d$. This can also be thought of as a depth-2 linear neural network with $d$ hidden units.
% see Figure \ref{fig:nn-picture}.
For training data $\crl{(x_i,y_i)\in\R\times\R}_{i=1}^n$, this results in the loss
\begin{equation}
\min_{\a,\b\in\R^d}\bar{L}(\a,\b) \quad := \quad \frac{1}{2n} \sum_{i=1}^n \prn*{\a^\top \b x_i - y_i}^2.
\end{equation}
This objective is equivalent---by rescaling and subtracting a constant---to the even simpler loss\footnote{See Lemma \ref{lem:simplify-objective} in Appendix \ref{appendix:objective} for a simple proof.} 
\begin{equation}
\label{problem:def}
\min_{\a,\b\in\R^d}L(\a,\b) \quad := \quad \frac{1}{2}\prn*{\a^\top\b - \Phi}^2.
\end{equation}
In what follows, we focus on this formulation and assume $\Phi \geq 0$ w.l.o.g.~for simplicity and clarity. 

Despite its simplicity, the objective \eqref{problem:def}, which has also been studied by prior work \citep{lewkowycz_large_2020,wang_large_2022,chen_beyond_2023,ahn_learning_2024,xu_three_2024}, is a useful object of study because it has a number of qualitative similarities to more complex and realistic problems like deep learning training objectives. First, it has similar high-level properties---the problem \eqref{problem:def} is non-convex and highly under-determined because the set of minimizers constitutes the ($2d-1$)-dimensional hyperboloid in $\R^{2d}$ that solves $\atb = \Phi$. It also exhibits some of the same symmetries as realistic neural networks; for example, $\atb$ is invariant to swapping ``neurons'' $(\a_i,\b_i) \leftrightarrow (\a_j,\b_j)$ or to rescaling $(\a_i,\b_i) \to (c\a_i,c^{-1}\b_i)$. More importantly, the dynamics when optimizing \eqref{problem:def} with gradient descent are qualitatively similar to the dynamics of training more complex models \citep[see, e.g.,][]{xu_three_2024}. Simultaneously, the problem \eqref{problem:def} is simple enough that we can provide a detailed and nearly comprehensive characterization of several different aspects of training.

% \begin{figure}
% \centering
% \drawnn
% \caption{A univariate linear neural network with $d=4$ hidden units.}
% \label{fig:nn-picture}
% \end{figure}

\subsubsection*{Main Contributions}
We analyze the discrete (finite-step) gradient descent (GD) trajectories for the above objective and provide the following results:

\paragraph{1. Convergence of Gradient Descent.}
Prior work \citep{wang_large_2022} shows that GD converges to a global minimum from almost every initialization, despite the non-convexity of \eqref{problem:def}. We strengthen this claim by proving \emph{linear convergence} despite no PL-condition: GD converges at a linear rate, with explicit dependence on the stepsize \(\eta\), the initial parameters \((\mathbf{a}(0), \mathbf{b}(0))\), and the target value \(\Phi\). Additionally, we identify several distinct \emph{phases} in the training dynamics, determined by the relationships among the stepsize \(\eta\), the parameter scale \(\scale := \|\mathbf{a}\|^2 + \|\mathbf{b}\|^2\), and the residual \(\residual := \mathbf{a}^\top \mathbf{b} - \Phi\). Some of these phases align with the so-called Edge of Stability (\textsc{EoS}) phenomenon \citep{cohen_gradient_2021}, where GD continues to reduce the loss even when the Hessian’s largest eigenvalue exceeds \(2/\eta\).

% \paragraph{1. Convergence of gradient descent.}
% Despite the non-convexity of the optimization problem \eqref{problem:def}, prior work has shown that GD converges to a global minimum from a.e.~initialization \citep{wang_large_2022}. We show that, in fact, it converges at a linear rate.
% %that depends on the stepsize, initialization of the parameters, and solution value $\Phi$. 
% In addition, we identify several phases that depend on the relationship between the stepsize, $\eta$; the scale of the parameters, $\scale := \nrm{\a}^2 + \nrm{\b}^2$; and the residuals, $\residual := \atb - \Phi$. Several of these phases are closely related to the so-called Edge of Stability (\textsc{EoS}) phenomenon \citep{cohen_gradient_2021}, where gradient descent decreases the objective (although non-monotonically) despite the largest eigenvalue of the objective's Hessian matrix being larger than the critical threshold $2/\eta$. 

% \begin{enumerate}
% \item Despite the non-convexity of the problem, gradient descent on \cref{problem:def} converges linearly to a global minimizer of the loss for any sufficiently small stepsize. We have the fastest, most general, and most explicit bound \todo{explain that compared to related work, the allowable stepsize is bigger, with a table??}
% \item We completely characterize the dynamics and the speed of convergence given any learning rate size and any initialization. We find that there are different phases for it, given the size of the learning rate. We are the first to our knowledge to unveil this structure.
% \end{enumerate}

\paragraph{2. Location of Convergence.}
Beyond convergence speed, we also describe \emph{which} global minimizers GD selects out of the many solutions satisfying \(\mathbf{a}^\top \mathbf{b} = \Phi\). We show that GD \emph{implicitly regularizes} the \emph{imbalance} 
\[
Q \;=\;\sum_{i=1}^d\bigl|\mathbf{a}_i^2 - \mathbf{b}_i^2\bigr|,
\]
consistently decreasing \(Q\) whenever \(\eta\) is not too large, whereas \emph{gradient flow} (GF) conserves \(Q\). This discrepancy implies that discretizing the training steps (i.e., using GD) can produce strictly \emph{flatter} solutions compared to GF—since the solution’s ``sharpness,'' or top Hessian eigenvalue, is tied to \(\|\mathbf{a}\|^2 + \|\mathbf{b}\|^2\). In particular, larger stepsizes promote stronger regularization and lower sharpness. While this sharper–flatter distinction has no effect on this model’s \emph{prediction} (the function is the same at all global minima), it can be relevant for generalization in more complex settings \citep{hochreiter_flat_1997,keskar_large-batch_2016,smith_bayesian_2018,park_effect_2019}. 
There is a large body of work, indeed, in other contexts showing that flatter minima of the loss tend to generalize better \citep{hochreiter_flat_1997,keskar_large-batch_2016,smith_bayesian_2018,park_effect_2019}, and our analysis shows how the self-multiplying structure of \eqref{problem:def} tends to regularize the sharpness.

\paragraph{Key Implications.}
Overall, these findings yield a near-complete picture of how the “product-parameterized” objective \(\tfrac12\bigl(\mathbf{a}^\top\mathbf{b} - \Phi\bigr)^2\) behaves under GD.
Beyond their intrinsic interest, we believe these results yield useful intuition for large-step training and implicit regularization for general neural networks.
In particular, 
\begin{enumerate}[leftmargin=1em,itemsep=0.2em,parsep=0.1em]
    \item[(i)] \textbf{GD Regularizes More Than GF.} \ 
    On this model, discretized GD strictly reduces the parameter imbalance \(Q\) (and thus the norm) beyond what continuous gradient flow achieves, with bigger step sizes regularizing more. However, \(Q\) does not necessarily vanish to zero, so solutions are typically not maximally regularized.

    \item[(ii)] \textbf{GF Is Not Always a Good Proxy.} \ 
    Since gradient flow preserves \(Q\), while GD decreases it, using GF to approximate GD can be misleading—particularly at moderate or large \(\eta\). Our results \textit{quantify} such discrepancies.

    \item[(iii)] \textbf{Convergence Speed vs.~Regularization Trade-off.} \ 
    Surprisingly, stronger implicit regularization of \(Q\) generally slows the overall convergence rate, and vice versa. This trade-off is closely linked to the \textsc{EoS} regime, where a large \(\eta\) can flatten the final solution but may cause non-monotonic or slower convergence.
\end{enumerate}

% Overall, these findings yield a near-complete picture of how the “product-parameterized” objective \(\tfrac12\bigl(\mathbf{a}^\top\mathbf{b} - \Phi\bigr)^2\) behaves under GD.
% Beyond their intrinsic interest, we believe these results yield useful intuition for large-step training and implicit regularization for general neural networks.

% \paragraph{Technical Overview.}
% A central insight is that each GD step scales the imbalance \(Q\) by the factor \(\bigl|1 - \eta^2\,\residual(t)^2\bigr|\). Consequently, for \(\eta < \sqrt{2}/|\residual(t)|\), \(Q\) decreases steadily. Meanwhile, the loss \(L\) does not satisfy the Polyak–\L{}ojasiewicz (PL) condition \citep{polyak_gradient_1963} globally (due to the saddle at the origin), but \emph{does} satisfy a form of PL \emph{along the GD trajectory} (Definition~\ref{def:PLAT}). This yields linear convergence to a global minimizer, provided the training path does not approach the origin too closely. Interestingly, the PL constant depends on the smallest \(\scale(t)\) encountered during training, which in turn is closely tied to how much \(Q\) has shrunk. Thus, more aggressive regularization of \(Q\) can force \(\scale(t)\) to become smaller, slowing the convergence rate in a self-reinforcing cycle. This tension explains the interplay between large stepsizes (leading to flatter solutions) and potentially slower or more oscillatory convergence in the Edge-of-Stability regime.

\paragraph{Technical Overview.} The key to our analysis is the following pair of observations. On the one hand, gradient descent iterations change the imbalance like $Q(t+1) = |1-\eta^2\residual(t)^2|Q(t)$, so the imbalance decreases throughout optimization for $0 < \eta < \nicefrac{\sqrt{2}}{\abs{\residual(t)}}$. At the same time, the objective $L$ does \emph{not} globally satisfy the Polyak-{\L}ojasiewicz (PL) condition \citep{polyak_gradient_1963} because the origin is a saddle point, but it \emph{does} satisfy a version of the PL condition along the GD trajectory, which is sufficient to prove linear convergence of GD to a global minimizer. Interestingly, the PL constant along the GD trajectory, which controls the speed of convergence, is equal to the smallest value of $\scale(t)$ encountered along the way, which is itself approximately equal to the value of $Q(t)$ at the first time that $\a(t)^\top\b(t) > 0$. Thus, the \emph{stronger} the implicit regularization of $Q$, the \emph{slower} the convergence of GD, and vice versa, which puts these goals directly at odds.

\section{Related Work}

A large body of research has shown empirically that training neural networks with larger learning rates tends to lead to better generalization \citep{lecun_efficient_2002,bjorck_understanding_2018,li_towards_2019,jastrzebski_break-even_2020}. However, in classical settings, convergence can only be guaranteed when the stepsize is small enough that $\scale_{\max}(\nabla^2 L) < \nicefrac{2}{\eta}$ throughout optimization \citep{bottou_optimization_2018}. Nevertheless, a recent line of work starting with \citet{cohen_gradient_2021} observed that when training neural networks, the maximum eigenvalue of the Hessian, or ``sharpness'', tends to grow throughout training until it reaches, or even surpasses the critical $\nicefrac{2}{\eta}$ threshold. But rather that diverging, the loss continues to decrease (non-monotonically) while the sharpness continues to hover around $\nicefrac{2}{\eta}$, which is referred to as the Edge of Stability (\textsc{EoS}) phenomenon. Understanding more deeply the training of neural networks with large stepsizes is of great interest.

Problems closely resembling \eqref{problem:def} have been studied previously. \citet{ahn_learning_2024} study losses of the form $\ell(ab)$ with $a,b\in\R$ and $\ell$ any convex, Lipschitz, and even function. The assumption that $\ell$ is even means it is minimized at zero (this is analogous to $\Phi=0$ in our case), and they prove convergence to zero from any initialization with any stepsize, but without a rate. However, this result relies crucially on both the loss being Lipschitz and minimized at zero. This is not surprising---we know that GD diverges on realistic objectives when the stepsize is too large. They also show that the limit point of gradient descent satisfies $\abs{a_\infty^2 - b_\infty^2} \approx \min\crl{\nicefrac{2}{\eta},\abs{a_0^2 - b_0^2}}$, i.e.~the imbalance between the weights is implicitly regularized down to the level of $\nicefrac{2}{\eta}$. \citet{chen_beyond_2023} study \eqref{problem:def} with scalar $a,b\in\R$ and prove that the limit point of GD satisfies $a-b \to 0$ when the stepsize is chosen slightly too large for convergence to any minimizer $ab=\Phi$ to be possible. This is qualitatively similar to our work, but they intentionally choose a too-large stepsize in order to highlight the implicit regularization of the imbalance, while we provide conditions on $\eta$ under which convergence to a minimizer and some amount of regularization happen simultaneously. 

In a related study, \citet{xu_three_2024} explore the continuous dynamics of gradient flow using the exact same model discussed here. They demonstrate that the dynamics unfold along a one-dimensional curve, with the location of convergence distinctly defined by conserved quantities. Contrary to their findings, our research reveals this is not the case for gradient descent, highlighting the danger of relying excessively on continuous models to understand discrete non-convex optimization dynamics.

In the most closely related work, \citet{wang_large_2022} study the exact objective \eqref{problem:def} and show that gradient descent using any stepsize up to $\eta \lesssim \nicefrac{4}{\textrm{sharpness}}$---approximately twice as large as the classical threshold of $\nicefrac{2}{\textrm{sharpness}}$---eventually converges to a minimizer, but without a rate. They also show some level of implicit regularization of $\scale$, e.g.~at convergence $\scale \leq \frac{2}{\eta}$. In comparison, we provide an explicit convergence rate for GD and give a more detailed connection between this rate and the implicit regularization.

Finally, many papers have studied other models such as matrix factorization or linear neural networks \citep{saxe_exact_2014,arora_convergence_2019,gidel_implicit_2019,tarmoun_understanding_2021,xu_linear_2023,nguegnang_convergence_2024}, which are more faithful representations of realistic neural networks, but they are also much more difficult to analyze. Due to this difficulty, these results often only apply to gradient flow, or to GD with a very small learning rate, or to GD under additional, hard to interpret assumptions. In this work, we focus on the problem \eqref{problem:def} in order to obtain a simpler, easier to interpret set of results.

\section{Preliminaries}
\label{sec:notation}

We study gradient descent (GD) on the quadratic loss
\begin{equation}\label{eq:loss-def}
    L(\mathbf{a},\mathbf{b}) 
    \;=\; \tfrac12\,\bigl(\mathbf{a}^\top \mathbf{b} \;-\; \Phi\bigr)^2,
    \quad \mathbf{a},\mathbf{b}\,\in\mathbb{R}^d,
\end{equation}
where \(\Phi \ge 0\). The discrete updates take the form
\begin{align}
\!\!
\begin{bmatrix}
\mathbf{a}(t+1)\\[2pt]\mathbf{b}(t+1)
\end{bmatrix} 
&=\,
\begin{bmatrix}
\mathbf{a}(t)\\[2pt]\mathbf{b}(t)
\end{bmatrix}
\,-\,\eta\nabla L\bigl(\mathbf{a}(t), \mathbf{b}(t)\bigr) \nonumber\\
&=\,
\begin{bmatrix}
\mathbf{a}(t)\\[2pt]\mathbf{b}(t)
\end{bmatrix}
\,-\,\eta\bigl(\mathbf{a}(t)^\top\mathbf{b}(t)-\Phi\bigr) \!
\begin{bmatrix}
\mathbf{b}(t)\\[2pt]\mathbf{a}(t)
\end{bmatrix} \!,
\label{eq:parameter-gd-dynamics}
\end{align}
but tracking the evolution of \(\mathbf{a},\mathbf{b}\) directly can be unwieldy. Instead, we reparametrize via three auxiliary quantities that more cleanly describe the training dynamics.

\paragraph{Residuals.}
Define the residual % $\varepsilon \;:=\; \mathbf{a}^\top \mathbf{b} \;-\;\Phi.$
\[
\residual 
\;:=\; \mathbf{a}^\top \mathbf{b} \;-\;\Phi.
\]
This scalar measures the distance to the manifold of global minima \(\{\mathbf{a}^\top\mathbf{b} = \Phi\}\). Its sign and magnitude will be key to understanding convergence.

\paragraph{Norm of the parameters.}
Let % $\scale \;:=\; \|\mathbf{a}\|^2 \;+\; \|\mathbf{b}\|^2$.
\[
\scale 
\;:=\;
\|\mathbf{a}\|^2 
\;+\;
\|\mathbf{b}\|^2.
\]
Crucially, the Hessian of \eqref{eq:loss-def} at a point \((\mathbf{a}, \mathbf{b})\) satisfying \(\mathbf{a}^\top \mathbf{b} = \Phi\) is
\begin{equation}
\nabla^2 L(\mathbf{a}, \mathbf{b})
\;=\;
\begin{bmatrix}
\mathbf{b}\\[2pt]\mathbf{a}
\end{bmatrix}
\!\begin{bmatrix}
\mathbf{b}\\[2pt]\mathbf{a}
\end{bmatrix}^\top,
\end{equation}
whose top eigenvalue (sharpness) is exactly \(\scale\). In particular, the flattest possible minimizer has \(\scale = 2\Phi\), achieved when \(\mathbf{a}\) and \(\mathbf{b}\) coincide up to a sign (due to Cauchy-Schwarz inequality).
Moreover, \(\scale\) governs how the residuals evolve under \eqref{eq:parameter-gd-dynamics}. Indeed, we have that
\begin{equation}
\label{eq:residuals-update}
\residual(t+1)
\;=\;
\residual(t)\,\Bigl[\,1 \;-\;\eta\,\scale(t)\;+\;\eta^2\,\residual(t)\!\bigl(\residual(t)+\Phi\bigr)\Bigr].
\end{equation}
Hence, the term \(-\,\eta\,\scale(t)\) is the principal driver for reducing \(\residual(t)\). In turn, \(\scale(t)\) evolution is governed by \(\residual(t)\)
\begin{equation}
\label{eq:scale-update}
\scale(t+1)
\;=\;
\bigl[\,1 + \eta^2\,\residual(t)^2\bigr]\;\scale(t)
\,-\,
4\eta\residual(t)\bigl(\residual(t) + \Phi\bigr).
\end{equation}

\paragraph{The imbalance.}
Finally, define \(Q_i := \mathbf{a}_i^2 - \mathbf{b}_i^2\) for \(i = 1,\ldots,d\) and let
\[
Q 
\;:=\; 
\sum_{i=1}^d \bigl|\,Q_i\,\bigr|.
\]
This quantifies how “imbalanced” the individual components of \(\mathbf{a}\) and \(\mathbf{b}\) are. Under \emph{gradient flow}, each \(Q_i\) is conserved, but under \emph{discrete} GD we have
\begin{align}
Q_i(t+1) 
&=\;\bigl[\,1 \;-\;\eta^2\,\residual(t)^2\,\bigr]\;Q_i(t),
\label{eq:Qi-update}
\end{align}
Whenever \(\eta < \sqrt{2}\bigl/\!\bigl|\residual(t)\bigr|\), the term \(\bigl|\,1 - \eta^2\,\residual(t)^2\bigr|\) is in \((0,1)\), so \(Q\) strictly decreases as \(Q_i\), for all $i$, decreases in absolute value. This decline in \(Q\) represents a core \emph{implicit regularization} effect unique to discrete steps. Our analysis leverages \(Q\) in two main ways: (i)~lower-bounding \(\scale\) by \(Q\) to control convergence speed, and (ii)~characterizing which minimizer (among the infinitely many) the algorithm ultimately selects.

\medskip

Together, \(\residual(t), \scale(t)\), and \(Q(t)\) offer a more tractable viewpoint than tracking \(\mathbf{a}(t)\) and \(\mathbf{b}(t)\) directly. In subsequent sections, we use these reparametrized updates \eqref{eq:residuals-update}--\eqref{eq:Qi-update} to analyze both the \emph{speed} of convergence and the \emph{location} (norm/sharpness) of the solution.

\section{Location of Convergence}
\label{section:location}

We first address the question of \emph{which} minimizers (among the infinitely many satisfying $\mathbf{a}^\top \mathbf{b}=\Phi$) is selected by GD. Our main theorem shows that GD reduces the imbalance $Q_i := \mathbf{a}_i^2 - \mathbf{b}_i^2$, whereas GF keeps it constant.

\begin{theorem}\label{theo:loc}
Let $0 < \eta < \min\Bigl\{\tfrac{1}{2|\residual(0)|}, \tfrac{2}{\sqrt{\scale(0)^2 + 4\,\Phi^2}}\Bigr\}.$ 
Then, at the limit point of gradient descent
we have\footnote{A similar result holds for larger stepsizes, but its statement is more involved. See Appendix \ref{app:location}.}
\begin{equation*}
\begin{split}
&0 
\;<\;
|Q_i(0)| \,\exp \Bigl(- \tfrac{\sqrt{\eta}\,\residual(0)^2}{\Phi}\Bigr)
\;<\;
|Q_i(\infty)| \quad \text{and} \\
&|Q_i(\infty)|
\;<\;
|Q_i(0)| \,\exp\Bigl(- \eta^2 \!\sum_{t=0}^{\infty}\residual(t)^2\Bigr)
\;<\;
|Q_i(0)|,
\end{split}
\end{equation*}
for all $i \in \{1,2, \ldots, d\}$.
\end{theorem}

\noindent
The proof follows directly from iterating the update in Equation~\eqref{eq:Qi-update}, which governs the evolution of each $Q_i$ under GD. The first bound on the step size is needed to prove the lower bound on $|Q_i(\infty)|$, and the second to show rapid convergence (cf.\ Theorem \ref{theo:speed}).
A full proof is located in Appendix \ref{app:location}.  

\begin{figure}[ht]
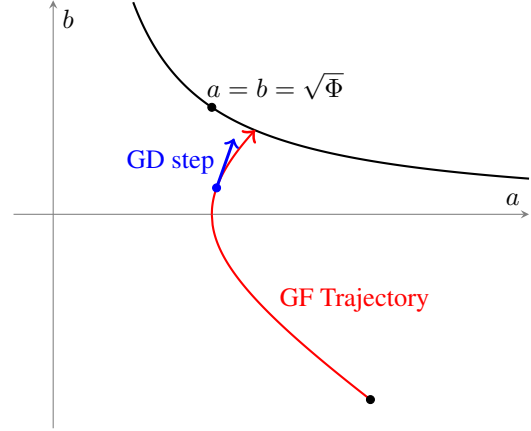

\centering
\gdupdatepicture
\caption{%
\label{Fig:GD_cuts}%
Case for $\a=a,\b=b \in \R$. Under gradient flow, the trajectory curves away from the origin, conserving $Q_i$. 
GD’s discrete step moves along the affine tangent space, shrinking $Q_i$.%
}
\end{figure}

Figure~\ref{Fig:GD_cuts} offers a geometric intuition: GF conserves the quantities $Q_i$ by curving away from the origin:
\[
\dot{Q_i} 
= 
2(\mathbf{a}_i\,\dot{\mathbf{a}}_i - \mathbf{b}_i\,\dot{\mathbf{b}}_i) 
= 
2\bigl(\mathbf{a}_i(-\,\residual\,\mathbf{b}_i) - \mathbf{b}_i(-\,\residual\,\mathbf{a}_i)\bigr) 
=
0.
\]
By contrast, the discretization error, introduced by the fact that gradient descent moves along the parallel vector to the curve, results in GD moving ``inward'' towards the line $\a=\b$, as illustrated by Figure \ref{Fig:GD_cuts}. This results in smaller imbalances $Q_i$, although it never reduces them to exactly $0$ (unless there is a step where $\eta = 1/\residual$ exactly).

\begin{center}
\textbf{Takeaway 1:} 
\emph{GD converges to a solution with \emph{strictly lower} imbalance than GF, 
though the imbalance never vanishes entirely.}
\end{center}

% Since GF is often used as a simpler analytical stand-in for GD, Theorem~\ref{theo:loc} warns that this approximation can be qualitatively misleading—particularly at moderate or large step sizes where a qualitatively different $Q$ implies qualitatively different convergence speeds.
Theorem \ref{theo:loc} describes an implicit regularization effect which is only due to the action of discretizing the dynamics.
Given that GF is frequently used as a simpler analytical stand-in for gradient descent in the literature, Takeaway 1 underscores the risks associated with over-relying on this approximation, potentially leading to inaccurate predictions about real-world behaviors.

\paragraph{Quantifying the Implicit Regularization.}
From Equation~\eqref{eq:Qi-update} and Theorem~\ref{theo:loc}, we can approximate $Q_i(\infty)$ when $\eta\,|\residual(t)|$ stays sufficiently small:
\begin{equation}
\begin{split}
    |Q_i(\infty)|
    &\;=\;
    |Q_i(0)|\cdot\prod_{t=0}^{\infty} \bigl|\,1 - \eta^2\,\residual(t)^2\bigr|\\
    &\;\approx\;
    |Q_i(0)|\cdot\exp\Bigl(-\,\eta^2\!\sum_{t=0}^{\infty}\residual(t)^2\Bigr).
\end{split}
\end{equation}
Hence, $Q_i(\infty)$ depends directly on how quickly $\residual(t)$ (and thus the loss) decreases. In Section~\ref{section:speed}, we show that under certain step-size conditions, $\residual(t)$ converges linearly, implying
\(
\sum_{t=0}^\infty \residual(t)^2 \,\approx\, \frac{\residual(0)^2}{\eta\,\mu},
\)
so the final imbalance $Q(\infty)$ experiences only a modest reduction. Conversely, in slower convergence regimes, $\sum_{t=0}^\infty \residual(t)^2$ can be very large, making $Q(\infty)$ significantly smaller than $Q(0)$. Thus, slower optimization can result in stronger regularization.

\section{Speed of Convergence}
\label{section:speed}

While \citet{wang_large_2022} already showed that gradient descent (GD) converges for this model initialized almost everywhere, we now establish an \emph{explicit rate} of convergence to a global minimizer.
We also establish in Proposition \ref{prop:GF} the exponential convergence of GF for \textit{every} initialization.

\begin{theorem}\label{theo:speed}
Let
\[
0 < \eta < \min\!\Bigl\{\tfrac{1}{2\,|\residual(0)|},\;\tfrac{2}{\sqrt{\scale(0)^2 + 4\,\Phi^2}}\Bigr\},
\]
and define 
\[
\bar{\eta} 
\;:=\;
\min\Bigl\{\,\eta,\; \tfrac{2}{\sqrt{\scale(0)^2 + 4\,\Phi^2}} - \eta\Bigr\}.
\]
Assume $Q(0)\neq 0.$\footnote{Note that we also handle separately the case $Q(t)\neq 0$ for some $t$ in Appendix~\ref{app:speed}.} Then for any $\delta>0$, there exists an iteration $T$ such that $L\bigl(\mathbf{a}(T),\mathbf{b}(T)\bigr)\le \delta$, and
\begin{align}
\label{eq:theo:speed:1}
T  
&\;\le\;
\mathcal{O}\bigg(\frac{\log(\a(0)^\top \b(0))}{\eta \Phi} + \frac{|\a(0)^\top \b(0)|}{\eta \Phi}
  \\
  \label{eq:theo:speed:2}
  &\quad+\;
  \frac{\log\!\bigl(\tfrac1\delta\bigr)}
       {\bar{\eta}\,Q(0)\,\exp \bigl(\min\{-\,\mathbf{a}(0)^\top \mathbf{b}(0),\,0\}\bigr) 
        + \bar{\eta}\,\Phi}
\biggr).
\end{align}
If instead
\[
\min\Bigl\{\tfrac{\sqrt{2}}{|\residual|},\;\tfrac{2}{\sqrt{\scale(0)^2 + 4\,\Phi^2}}\Bigr\}
\;<\;\eta\;<\;
\min\Bigl\{\tfrac{2}{|\residual|},\;\tfrac{2}{\scale} + \tfrac{2\,\residual\,(\residual+\Phi)}{\scale^3}\Bigr\},
\]
then GD converges but may do so at a \emph{logarithmically slow} rate\footnote{Meaning there exists an arbitrarily long phase of decay with the rate $\eta(t+1) - \eta(t) = \big(\substack{\text{small}\\\text{constant}}\big) \cdot \eta(t)^2$ which relates with the ODE $\dot x = -x^2$ which goes as $1/t$ instead of $\exp(-t)$.}.
\end{theorem}

\noindent
A detailed proof appears in Appendix~\ref{Appendix:PLAT}-\ref{app:speed}, while Section~\ref{section:proof_sketch} sketches the main ideas. The theorem shows that even when $\eta$ is relatively large (but below certain thresholds), GD converges \emph{linearly}, up to a constant additive term reflecting how long it takes to escape the region around its initialization. Concretely, the convergence can be split into two phases:  
\begin{itemize}[leftmargin=1em,itemsep=0.2em,parsep=0.1em]
    \item \textbf{Equation \eqref{eq:theo:speed:1}:} A phase where $|\residual| > |\Phi|$, during which convergence may slow and the trajectory risks nearing the saddle at the origin, although with probability 1 it avoids it, generally very quickly. We show that the speed of escaping of this phase is exponential up to log factors.
    \item \textbf{Equation \eqref{eq:theo:speed:2}:} A phase where the iterates are sufficiently close to the manifold of minima (\(\mathbf{a}^\top\mathbf{b} \approx \Phi\)), yielding exponential convergence once the dynamics remains in regions A and B (see below). This is the setting in which theorems of convergence of linear networks as the one in \cite{arora_convergence_2019} apply.
\end{itemize}
Hence, the overall convergence speed is exponential.
It is important to note that the rate of convergence in both phases depends on the unbalance $Q$ at the iteration in which we switch phase. This implies that if $Q$ is smaller, convergence happens slower. At the same time, in the previous section we established that if convergence is slower then the implicit regularization on $Q$ is stronger. This unveils an important trade off in the dynamics:

\begin{center}
\textbf{Takeaway 2:}
\emph{Stronger implicit regularization in the first phase of the training slows convergence. Generally, faster speed of training slows implicit regularization and slower speed of convergence imply stronger implicit regularization.}
\end{center}

This speed–regularization trade-off is illustrated in Figure~\ref{fig:experiments}. Even then, we observe an inverse relationship between speed and the strength of $Q$-regularization.

\subsubsection*{Edge-of-Stability Case}
Recent work \citep{cohen_gradient_2021} on training neural networks with MSE shows that for full-batch gradient descent, the Hessian’s largest eigenvalue generally hovers just above $2/\eta$ without causing divergence. For linear gradients, $\eta>\tfrac{2}{\scale}$ ordinarily implies divergence (e.g., the 1D parabola case), yet real neural networks manage to converge. The model we analyze offers a possible explanation: the product structure, combined with discrete updates, still allows convergence for $\eta>\tfrac{2}{\scale}$, but more slowly. 

Moreover, our work hint to a possible important and surprising benefit of training at the Edge of Stability: we prove that the slower the convergence and the larger $\eta$, the smaller the inbalance $Q$. Hence, training at the edge of stability albeit at the cost of longer or more oscillatory paths,  may enhance implicit regularization,.

\section{Proof Sketch}
\label{section:proof_sketch}
In what follows we present informally the proof of Theorem \ref{theo:speed}.
Note that Equation \eqref{eq:residuals-update} implies approximately that
\begin{equation}
\residual(t+1) \approx (1-\eta \scale(t))\residual(t).
\end{equation}
We show in Appendix \ref{appendix:sharpness} that if at initialization $\eta < \tfrac{2}{\sqrt{\scale(0)^2 + 4\,\Phi^2}}$ this will be the case throughout the trajectory. This implies that $\eta \scale(t)$ is upper bounded by a quantity strictly smaller than 2 and decreases.
Precisely,
\begin{lemma}
\label{lemma:bound_norm}
    Let $\eta \leq 1/\residual(0)$ and assume that $|\residual(t)|$ is monotonically decreasing along the trajectory of GD.
    Then $\scale(t)$ is bounded for all steps $t$ by
    \[
    \scale
    % \quad \leq \quad
    % \sqrt{ \scale(0)^2 - 8 \residual(0) (\residual(0) + \Phi) + 4 \residual(0)^2 } 
    \quad \leq \quad
    \sqrt{ \scale(0)^2 + 4\Phi^2} 
    .
    \]
    Analogously, along the trajectory of GF, $\scale(t)$ is bounded for all steps $t$ by the same quantity.
\end{lemma}
Instrumental to prove this theorem is the observation that:
\begin{lemma}
\label{lemma:conserved_norm}
    The quantity $\alpha := \scale^2 - 8\residual (\residual + \Phi) + 4\residual^2$ is conserved by the gradient flow on $L$, and it is reduced by gradient descent as long as $\eta \leq \min\{1/\residual(0), 2/\scale\}$, at every step $t$ by the quantity
    \[
    \alpha(t+1)
    \ = \ 
    \alpha(t)
    \ - \ \underbrace{2 \eta^2 \residual(t)^2 Q(t)^2 \big|1 - \eta^2 \residual(t)^2 \big|}_{\text{positive when }\eta \leq 1/\residual(t)}.
    \]
\end{lemma}
This already roughly establishes convergence somewhere for big learning rates, but not its speed or location.

\begin{figure}[ht]
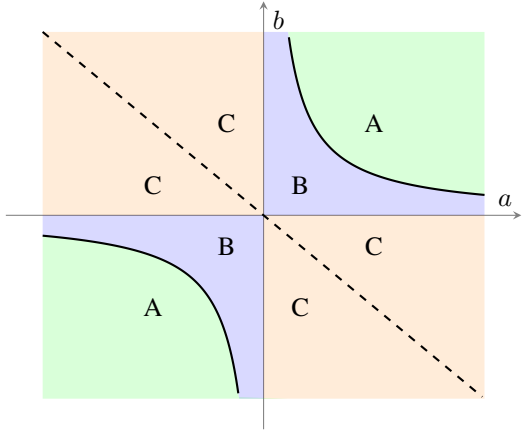

\centering
\maketikzabplot
\caption{%
\label{Fig:regions}%
Schematic of GD behaviors in three different regions: (A) $\residual > 0$, (B) $\residual < 0 < \mathbf{a}^\top\mathbf{b}$, (C) $\mathbf{a}^\top\mathbf{b} < 0$. 
See text for details.%
}
\end{figure}

To establish the rate of convergence for small and moderate step sizes we now must ensure $\scale(t)$ never becomes too small, i.e., the dynamics stays away from the origin. The key idea here is to notice that $\scale \geq Q$ and studying the evolution of its size partitioning the parameter space into the three “regions” depicted in Figure~\ref{Fig:regions}. Precisely, we show that along the trajectory $\scale,Q$ decreases but slowly enough to imply linear convergence both when your step stays within the region\footnote{Except to go from Region $C$ to Region $B$, the dynamics jumps between regions only for big learning rates.}:
\begin{enumerate}[leftmargin=1em,itemsep=0.2em,parsep=0.1em]
\item \textbf{Region~A:} $\residual>0$. 
   By Cauchy–Schwarz
   \[
   2 \Phi \ \leq \ 2\mathbf{a}^\top \mathbf{b} \ \leq \ \|\mathbf{a}\|^2+\|\mathbf{b}\|^2 \ = \ \scale,
   \]
   thus $\residual$ shrinks at a linear rate $(1-2\eta\Phi)$.
   
\item \textbf{Region~B:} $\residual<0$ but $\mathbf{a}^\top\mathbf{b}>0$. 
   Here, $\scale$ is small but each update \emph{increases} $\scale$, preventing it from collapsing to $0$. We prove $\scale(t)\ge \scale(\tau)$ for $\tau$ the time of entry into Region~B. This implies that $\residual$ shrinks at a linear rate of at least $(1-2\eta\scale(\tau))$ and $\scale(\tau)\ge Q(\tau)$.

\item \textbf{Region~C:} $\mathbf{a}^\top\mathbf{b}<0$. 
   Equation \eqref{eq:Qi-update} implies that the imbalance $Q$ evolves via $Q(t+1)=|1-\eta^2\,\residual(t)^2|\,Q(t)$. Equation \eqref{eq:residuals-update} implies that $\residual$ evolves roughly as $\residual(t+1)\approx (1-\eta\,Q(t))\,\residual(t)$. 
   Exiting Region~C requires $\residual$ crossing $-\,\Phi$.
   Since $Q$ shrinks slower than $\residual$ we establish in Appendix \ref{appendix:mu_cases} that roughly in $\mathcal{O}(\eta^{-1})$ steps 
   $Q$ only decreases of size $(1 - \eta)$ ensuring the dynamics is distant from the origin and thus $\residual$ change of $O(1)$ exiting Region~C towards Region~B\footnote{This is the most technical step of the proof.}.
\end{enumerate}
   
Putting all this together, we see that defining $\tau$ is the step in which $\a,\b$ are the closest to the origin, $\tau = \argmin_{t} \scale(t)$, we established that while $Q(t)$ decreases in $t$, $Q(\tau)$ never becomes vanishingly small, ensuring a positive lower bound to $\scale(t)$ and the speed of convergence. This yields linear convergence with rate $(1-\eta Q(\tau))$.

\begin{figure*}[t]
\centering
\includegraphics[width=0.49\linewidth]{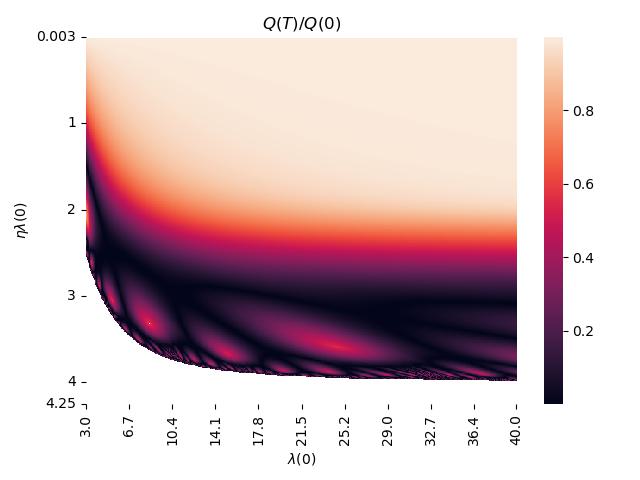}
\includegraphics[width=0.49\linewidth]{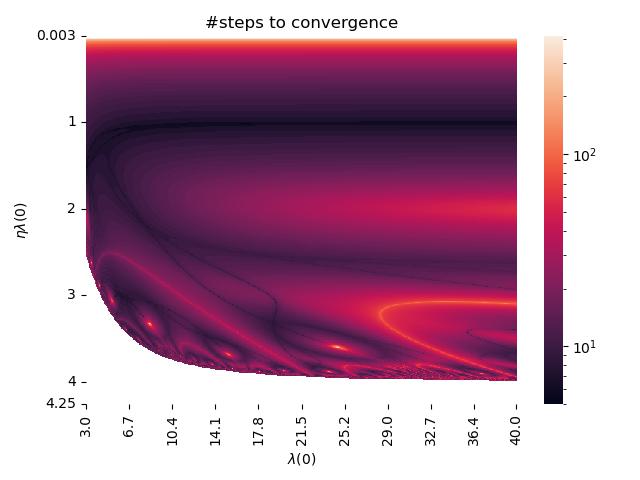}
\caption{\label{fig:experiments}
Gradient descent on \eqref{problem:def} with $\Phi=1$ and various step sizes and initial scales. 
Left: Ratio $Q(T)/Q(0)$ showing how little $Q$ changes when $\eta \,\scale(0)$ is small and how strongly it is reduced for large $\eta \,\scale(0)$. 
Right: The time to converge to a small residual, illustrating slower convergence in cases with stronger $Q$-regularization. The chaotic behavior appears when $\eta \geq 1/\residual$.}
\end{figure*}

\section{Conclusion}
In this paper, we analyzed the gradient descent dynamics of a depth-2 linear neural network, offering a simplified model to explore training behaviors observed in more complex neural networks. Our key technical contributions are:
\begin{enumerate}
    \item \textbf{Linear convergence with large step sizes:} We demonstrated that gradient descent converges at a linear rate to a global minimum, even with larger-than-expected step sizes—up to approximately $2/\textrm{sharpness}$. For even larger step sizes, convergence can still occur, but much more slowly. See Section \ref{section:speed}.
    
    \item \textbf{Location of convergence:} We characterized the solution reached by gradient descent, showing that it implicitly regularizes the parameter imbalance and sharpness, leading to a lower norm solution compared to gradient flow. Notably, as the step size increases, the implicit regularization effect strengthens, flattening the solution. See Section \ref{section:location}.
\end{enumerate}
The key implications of our results are that
\begin{enumerate}
    \item[i.] \textbf{GD always regularizes more than GF:}
    Gradient descent converges to a solution with lower imbalance than gradient flow, but the imbalance always remains non-zero. The solution is still suboptimal from this perspective.  See Section \ref{section:location}.

    \item[ii.] \textbf{GF is not always a good approximation of GD:}
    We prove that even in a very simple model, gradient flow dynamics are inherently different from gradient descent. In particular, our results can be used as a proof that the common use of GF as a theoretical tool for understanding GD is not always well founded.
    See Section \ref{section:location}.
    
    \item[iii.] \textbf{Trade-off Between Speed and Regularization:} Our analysis uncovered a trade-off between the convergence rate and the degree of implicit regularization. See Section \ref{section:speed}. Training at the edge of stability, while slower, induces additional regularization, which may be beneficial for generalization. See Section \ref{section:speed}.
\end{enumerate}
Our findings thus provide insight into different step sizes affect neural network training dynamics and its potential benefits for regularization in more complex models.

% \paragraph{Limitations.}
% Our model is very simple and it certainly will not capture all of the relevant properties of realistic neural networks. However, we believe it does capture some of them. As an example, a lot of recent work, e.g. \cite{ahn_learning_2024,wang_large_2022} has studied the Edge of Stability (\textsc{EoS}) \cite{cohen_gradient_2021} phenomenon. Our model is enough to, at least partially, study what happens in that regime, see Section \ref{section:\textsc{EoS}}. 
% Moreover, it is also true that other articles may study the same model (or more complex/realistic models similar to ours), none of them include analysis at this level of detail, and therefore they do not arrive at our conclusions, which we think are interesting and could form the basis for additional study. Certainly, we would like to extend our analysis to more complex models, but this presents a significant technical challenge that must be left for future work.

\paragraph{Future work:} In this work, we studied the model \eqref{problem:def} because its simplicity allows for a detailed analysis that leads to the useful conclusions detailed above. However, there are several possible extensions of these results that could lend additional insights. For example, it would be interesting to study the case of vector-valued inputs, deeper models, and non-linear models that use ReLU or other activation functions. In addition, we are interested to know how our results would be impacted by using stochastic gradient descent in rather than exact gradient descent.

\bibliographystyle{icml2025}
\bibliography{bibliography_ICLR25_GD}

\newpage
\appendix
\onecolumn

\section{On the Objective}\label{appendix:objective}

\begin{lemma}\label{lem:simplify-objective}
For any $\a, \b, \crl{(x_i,y_i)}_{i=1}^n$, 
\[
\bar L(\a,\b) \ = \ \frac{\sum_{i=1}^nx_i^2}{2n}(\a^\top\b - c)^2 + \textrm{Const}
\ = \ 
\left[\frac{1}{n}\sum_i x_i^2\right]  L(\a,\b) + \textrm{Const}.
\]
where $c = \frac{\sum_{i=1}^nx_iy_i}{\sum_{i=1}^nx_i^2}$ and $\textrm{Const} = \frac{1}{2n}\prn*{\sum_{i=1}^ny_i^2 - \frac{\prn*{\sum_{i=1}^nx_iy_i}^2}{\sum_{i=1}^nx_i^2}}$ does not depend on $\a,\b$.
\end{lemma}
\begin{proof}
Let $\x$ denote the vector whose $i$th entry is $x_i$, and let $\y$ denote the vector whose $i$th entry is $y_i$. Then we can write
\begin{multline}
\bar L(\a,\b) 
= \frac{1}{2n}\nrm*{\a^\top \b \x - \y}^2
= \frac{1}{2n}\prn*{(\a^\top\b)^2\nrm{\x}^2 -2\a^\top\b\inner{\x}{\y} + \nrm{\y}^2} \\
= \frac{\nrm{\x}^2}{2n}\prn*{(\a^\top\b)^2 -2\a^\top\b\frac{\inner{\x}{\y}}{\nrm{\x}^2} + \frac{\nrm{\y}^2}{\nrm{\x}^2}} \\
= \frac{\nrm{\x}^2}{2n}\prn*{\prn*{\a^\top\b - \frac{\inner{\x}{\y}}{\nrm{\x}^2}}^2 + \frac{\nrm{\y}^2}{\nrm{\x}^2} - \frac{\inner{\x}{\y}^2}{\nrm{\x}^4}} 
\end{multline}
Rewriting this in terms of the $x_i$'s and $y_i$'s completes the proof.
\end{proof}

Note, thus, that all our proofs work on $\bar L$, we thus have to rescale $\residual, \scale, Q, \eta$ accordingly. Precisely,
\begin{equation}
\begin{split}
    \scale &\curvearrowleft \left[\frac{1}{n}\sum_i x_i^2\right] \scale, 
    \quad
    Q \curvearrowleft \left[\frac{1}{n}\sum_i x_i^2\right] Q, 
    \quad
    \residual \curvearrowleft \left[\frac{1}{n}\sum_i x_i^2\right] \residual, 
    \quad
    \text{and} \quad
    \eta \curvearrowleft \left[\frac{1}{n}\sum_i x_i^2\right] \eta.
\end{split}
\end{equation}
Analogously, note that if $\Phi<0$ nothing changes in the analysis of the dynamics. When $\a \neq -\b$ just change $\a$ to $-\a$ and apply the same analysis as before.

\section{From the Residuals to the Loss}
First note that if $\residual$ converges exponentially to zero, then loss $L$ converges exponentially to its minimum.
\begin{lemma}
\label{lemma:conv:loss}
    Assume $|\residual(k)|$ converges linearly fast with rate $(1 - \eta \mu) < 1$.
    Then $L$ converges linearly fast with rate $(1 - \eta \mu)^2$.
    In particular, let $\delta>0$, the loss $L$ is smaller than $\delta$ in a number of steps $t$ that satisfies
    \[
    t \quad \leq \quad \frac{\log L_0 - \log(\delta)}{\eta \mu}.
    \]
\end{lemma}
Indeed note that for how we defined $\residual$ we have that $L = \residual^2$, thus $L(k+1) = |\residual(k+1)| \leq (1 - \eta \mu)|\residual(k)|^2$.
Note that this lemma allows us to deal with the convergence of $\residual$ instead of the convergence of $L$ and infer the convergence of $L$. Indeed, if the residuals $\residual$ converge linearly with rate $(1 - \eta \mu) < 1$, then the time it takes to converge is such that $\sqrt{\delta} \geq (1-\eta \mu)^t \sqrt{L_0}$ which is
\begin{equation}
    t \ \leq \ \frac{\log L_0 - \log(\delta)}{-\log(1-\eta \mu)}
    \ \leq \ \frac{\log L_0 - \log(\delta)}{\eta \mu}.
\end{equation}
From now on we will deal with convergence of residuals only.

\section{Bounding the final sharpness}
\label{appendix:sharpness}

\subsection{Size of $\scale$ for Gradient Flow}
Note that we can characterize the norm $\scale_\infty$ found by gradient flow by noticing that 
\begin{lemma}
\label{lemma:lambda_GF}
    Along the gradient flow trajectory, the following quantity is conserved
    \[
    \scale^2 - 8 \residual (\residual + \Phi) + 4 \residual^2.
    \]
\end{lemma}
This lemma proves the first part of Lemma \ref{lemma:conserved_norm}.
\begin{proof}
The gradient flow dynamics are described by
\begin{equation}
\begin{bmatrix}
\dot{\a}\\\dot{\b}
\end{bmatrix}
= -\nabla L(\a,\b) = - \residual\begin{bmatrix}
\b\\\a
\end{bmatrix}
\end{equation}
First, we compute 
\begin{align}
\dot{\scale} &= \frac{d}{dt}\brk*{\nrm{\a}^2 + \nrm{\b}^2} \\
&= 2\inner{\a}{\dot{\a}} + 2\inner{\b}{\dot{\b}} \\
&= -2\residual\inner{\a}{\b} - 2\residual\inner{\b}{\a} \\
&= -4\residual(\residual + \Phi)
\end{align}
and
\begin{align}
\dot{\residual} &= \frac{d}{dt}\brk*{\inner{\a}{\b} - \Phi} \\
&= \inner{\a}{\dot{\b}} + \inner{\dot{\a}}{\b} \\
&= -\residual\nrm{\a}^2 - \residual\nrm{\b}^2 \\
&= -\scale\residual
\end{align}
Finally, straightforward calculation confirms:
\begin{align}
\frac{d}{dt}\brk*{\scale^2 - 8 \residual (\residual + \Phi) + 4 \residual^2} 
&= 2\scale\dot{\scale} - 8\residual\dot{\residual} - 8 \dot{\residual}(\residual + \Phi) + 8\residual\dot\residual \\
&= 2\scale\dot{\scale} - 8 \dot{\residual}(\residual + \Phi) \\
&= 2\scale\prn*{-4\residual(\residual + \Phi)} - 8 \prn*{-\scale\residual}(\residual + \Phi) \\
&= 0
\end{align}
which completes the proof.
\end{proof}

\begin{lemma}
\label{lemma:lambda_GF_bound}
    $\scale(t)$ along the whole GF trajectory satisfies
    \[
    \scale(\infty)
    % \quad \leq \quad
    % \sqrt{ \scale(0)^2 - 8 \residual(0) (\residual(0) + \Phi) + 4 \residual(0)^2 } 
    \quad \leq \quad
    \sqrt{ \scale(0)^2 + 4\Phi^2} 
    .
    \]
\end{lemma}
\begin{proof}
    Note that
    \[
    \scale(\infty) = \scale(0)^2 - 8 \residual(0) (\residual(0) + \Phi) + 4 \residual(0)^2.
    \]
    Note that the maximum over $\scale = \scale(0)$ of $- 8 \residual(0) (\residual(0) + \Phi) + 4 \residual(0)^2$ is
    \begin{equation}
        4 \max_{\scale = \scale(0)} -\residual(\residual + 2\Phi)
    \end{equation}
    Is at $\residual = -\Phi$. This implies that for all the points with fixed $\scale$ the one with highest $\scale^2 - 8 \residual (\residual + \Phi) + 4 \residual^2$ is the one with $\residual = -\Phi$. Whatever was the initialization with a certain fixed scale, the solution found will have lambda smaller than $\scale^2 - 8 \residual (\residual + \Phi) + 4 \residual^2$, thus of $\sqrt{\scale(0)^2 + 4 \Phi^2}$.
    Next note that $\scale$ has positive derivative only when $\residual \in [-\Phi, 0)$. This implies that the sup for $\scale$ along the trajectory is either initialization or the solution.
\end{proof}

\subsection{Size of $\scale$ for Gradient Descent}

Surprisingly, we show here that if switch to gradient descent the quantity $\scale^2 - 8 \residual (\residual+\Phi) + 4\residual^2$ actually decreases to the second order in $\eta$.

\begin{lemma}
\label{lemma:lambda_GD}
    One step of gradient descent trajectory with step size $\eta > 0$, induces the following change in the quantity $\scale^2 - 8 \residual (\residual + \Phi) + 4 \residual^2$:
    \[
    \scale_1^2 - 8 \residual_1 (\residual_1 + \Phi) + 4 \residual_1^2
    \quad \curvearrowleft \quad 
    \scale^2 - 8 \residual (\residual + \Phi) + 4 \residual^2
    \ - \ 2 \eta^2 \residual^2 Q^2 (1 - \eta^2 \residual^2).
    \]
\end{lemma}
This lemma proves the second part of Lemma \ref{lemma:conserved_norm}.
\begin{proof}
Note that
\begin{equation}
\begin{split}
\scale_1^2 \ &= 
\big( (\a - \eta \residual \b)^2 + (\b - \eta \residual \a)^2 \big)^2
\\ &= \ \big( \scale (1 + \eta^2 \residual^2) - 4 \eta \residual (\residual + \Phi) \big)^2
\\ &= \ \scale^2 - 8 \eta \residual \scale (\residual + \Phi)
+ 2\eta^2 \residual^2 \scale^2 + \eta^4 \residual^4 \scale^2
+ 16 \eta^2 \residual^2 (\residual + \Phi)^2 
- 8 \eta^3 \residual^3 \scale (\residual + \Phi).
\end{split}
\end{equation}
Analogously
\begin{equation}
\begin{split}
4\residual_1^2 \ &= 
4\big( \residual (1 - \eta \scale) + \eta^2 \residual^2 (\residual + \Phi) \big)^2
\\ &= \ 4 \residual^2 - 8 \eta \residual^2 \scale
+ 4\eta^2 \residual^2 \scale^2
+ 8 \eta^2 \residual^3 (\residual + \Phi) 
- 8 \eta^3 \residual^3 \scale (\residual + \Phi)
+ 4 \eta^4 \residual^4 (\residual + \Phi)^2,
\end{split}
\end{equation}
and 
\begin{equation}
\begin{split}
-8\residual_1( \residual_1 + \Phi) \ &= 
-8\big( \residual (1 - \eta \scale) + \eta^2 \residual^2 (\residual + \Phi) \big( \Phi + \residual (1 - \eta \scale) + \eta^2 \residual^2 (\residual + \Phi) \big)
\\ &= \ 
-8\residual( \residual + \Phi)
+ 8 \eta\residual\scale (2 \residual + \Phi)
- 8 \eta^2 \residual^2 \scale^2
\\ & \quad - \ 
8 \eta^2 \residual^2 (\residual + \Phi) 
\big( (2\residual + \Phi) - 2\eta \scale \residual + \eta^2 \residual^2 (\residual + \Phi) \big).
\end{split}
\end{equation}
This (Lemma \ref{lemma:lambda_GF}) implies that the monomials of degree 1 in $\eta$ zeroes out, the monomial of degree 3 zeroes out too:
\begin{equation}
\begin{split}
\eta^3 \residual^3 \cdot 
\big( -8 \scale (\residual + \Phi) -8 \scale (\residual + \Phi)
+ 16 \scale (\residual + \Phi) \big) 
\quad = \quad 0.
\end{split}
\end{equation}
The monomials of degree 2 in $\eta$ are
\begin{equation}
\begin{split}
\eta^2 \residual^2 \cdot 
\big( \scale^2
(2 + 4 - 8) + (\residual + \Phi)^2 (16 - 8) + \residual ( \residual + \Phi) (8-8)  \big) 
\ = \ -2\eta^2 \residual^2 ( \scale^2 - 4(\residual + \Phi)^2).
\end{split}
\end{equation}
This is exactly equal to 
\[
-2\eta^2 \residual^2 Q^2.
\]
Analogously the monomial of degree $4$ in $\eta$ is
\[
2\eta^4 \residual^4 Q^2
\]
which completes the proof.
\end{proof}

\begin{lemma}
\label{lemma:lambda_GD_bound}
    Let $\eta < 1/|\residual(0)|$ and assume $|\residual(t)|$ is monotonically decreasing, along the GD trajectory
    \[
    \scale(t)
    % \quad \leq \quad
    % \sqrt{ \scale(0)^2 - 8 \residual(0) (\residual(0) + \Phi) + 4 \residual(0)^2 } 
    \quad \leq \quad
    \sqrt{ \scale(0)^2 + 4\Phi^2} 
    .
    \]
\end{lemma}
The proof follows as the one of Lemma \ref{lemma:lambda_GF_bound} by exchanging the equalities given by Lemma \ref{lemma:lambda_GF} with the inequalities given by Lemma \ref{lemma:lambda_GD}.

\begin{definition}[Maximal Sharpness $\bar \scale$]
    We denote by $\bar \scale$ and we call maximal sharpness the value
    \[
    \bar \scale 
    \quad := \quad
    \sqrt{ (\mednorm{\a(0)}^2 + \mednorm{\b(0)}^2)^2 + 4 \Phi^2 }.
    \]
\end{definition}

\section{PL Condition Along the Trajectories}

\subsection{Continuous Dynamics}
\label{appendix:GF}
\begin{proposition}
\label{prop:GF}
    The loss $L(\a,\b)$ equipped with gradient flow converges exponentially fast no matter the initialization. If $\a = -\b$ it converges to the saddle $\a=\b=0$. Otherwise, it converges to a global minimum.
    % The constant $\mu$ is $Q$ when $Q\neq 0$, $\Phi^2$ if $\a = -\b$, and $\min\left\{ 2\sum_{i \text{ s.t. } \a_1 \neq -\b_i} \a_1^2, \Phi^2 \right\}$ otherwise.
\end{proposition}
In the case of gradient flow the pairs $(\a,\b)$ along the trajectory satisfy a PL condition with $\mu(\a,\b) = \mednorm{\a}^2 + \mednorm{\b}^2$, indeed note that $L(\a,\b)$ satisfies
\[
% \left(ab-\Phi \right)^2 \left[\frac{1}{n}\sum x_i^2\right]^2 (\mednorm{u}^2 + \mednorm{v}^2) =
% \left[\frac{1}{n}\sum x_i^2\right]^2 
\left( \a^\top \b - \Phi \right)^2 (\mednorm{\a}^2 + \mednorm{\b}^2)
\ = \
\mednorm{\nabla L(\a,\b)}^2 \ = \ 
\mu(t) \cdot L(\a,\b).
\]
Note that for all $i$ the quantity $Q_i=\a_i^2 - \b_i^2$ is conserved along the trajectory, indeed 
\[
\frac{d}{dt} \left( \a_i(t)^2 - \b_i(t)^2 \right)
\quad = \quad
2\residual (\a_i\b_i -\a_i\b_i)
\quad = \quad
0.
\]
Thus we have that $Q(0) \neq 0$ is a lower bound to $\mu$ along the whole trajectory, we thus proved that
\begin{lemma}
    Let $\a(0),\b(0) $ such that $Q(0)\neq0$. The gradient flow starting from $\a,\b$ converges exponentially fast with rate at least $Q$ to the point $\a(\infty), \b(\infty)$ which satisfies that (i) $\a(\infty)^\top \b(\infty) = \Phi$ and (ii) for all $i$ that $Q_i(0)= Q_i(\infty)$ and $sign\big(\a_i(\infty) - \b_i(\infty)\big) = sign\big(\a_i(0) - \b_i(0)\big)$.
\end{lemma}
This lemma and the observation of what happens in the case of $Q=0$ in Section \ref{appendix:GD}, prove Proposition \ref{prop:GF}.

\subsection{Initialization such that $Q(0)=0$}
\label{appendix:GD}
Note that for a fixed initialization where $Q\neq 0$, if $\eta$ is such that there exists a step $k$ along the trajectory where $\eta \cdot (\a^\top \b - \Phi) = 1$ exactly, convergence happen to $\a=\b=0$ instead of the global minimum. Indeed, in this case, on the next step we have $\a(k+1) = -\b(k+1) = \a(k) - \b(k)$.
This implies that when $Q \neq 0$, for almost every $\eta$ in the allowed range we have $Q \neq 0$ along the whole trajectory, and as we prove, linear convergence to a global minimum.

We characterize below what happens in the case in which $Q = 0$ at some point along the trajectory.

For both GD and GF if at a certain point during the training (or at initialization) $\a$ and $\b$ are such that $Q(\a,\b) = 0$, then we are on the one dimensional manifold in which for every neuron $i$ we have $\a_i = \pm \b_i$.
\begin{itemize}[leftmargin=1em,itemsep=0.2em,parsep=0.1em]
        \item If $\a = -\b$ then the problem becomes $L = (\mednorm{\a}^2 + \Phi)^2$ and it converges to the minimum $\a = \b = 0$ of the modified loss $\tilde L = \mednorm{a}^2$. The gradient is such that 
        \[
        \mednorm{\nabla L(\a,\b)}^2 \ = \ 
        \left( \mednorm{\a}^2 + \Phi \right)^2 \cdot 2\mednorm{\a}^2 
        \ = \
        \mu(t) \cdot \tilde L(\a,\b).
        \]
        with $\mu(\a,\a) = 2(\mednorm{\a}^2+\Phi)^2 \geq 2\Phi^2>0$. Thus restricted to the manifold where the trajectory lies, we have a function satisfying the PL condition with $\mu \geq 2\Phi^2 > 0$. In this case both GD and GF converge linearly fast to the minimum along this manifold, i.e., the saddle point at the origin.
        \item If $\a \neq -\b$ and there exists a component $i$ such that $\a_i = \b_i$ instead the components $n_1<n$ components satisfying $\a_i = -\b_i$ will converge to $\a_i = \b_i = 0$, the the $n-n_1\geq 1$ other components will converge to the global minimum of $L$ with PL constant given by their norm at initialization $2\sum_{i \text{ s.t. } \a_1 \neq -\b_i} \a_1^2$.
        This implies that in this case we have convergence to $0$ for the neurons in which $\a_i = \b_i$ and the dynamics is as described in the rest of the manuscript for the other neurons in which $\a_i = -\b_i$. 
\end{itemize}
This implies that the manifold where the algorithms converge to the saddle is not just of measure zero, but it is precisely $\a = -\b$. Even in this case, we have linear convergence to the saddle, when the learning rate is smaller than $2/\scale$. 
In all the other cases, if $Q = 0$, we have a sub network where $\a = \b \neq 0$, thus the loss satisfies
\[
    \mednorm{\nabla L(\a,\b)}^2 \ = \ 
    \left( \mednorm{\a}^2 - \Phi \right)^2 \cdot 2\mednorm{\a}^2 
    \ = \
    \mu(t) \cdot \tilde L(\a,\b).
\]
with PL-condition $2\mednorm{\a}_2^2$, which is positive and bounded below by $2\mednorm{\a(0)}_2^2$ if initialized in Region B, and by $2\Phi^2>0$ if we initialized in Region A.

We thus have linear convergence either to the saddle at the origin or to a global minimum for $Q=0$. In the rest we abnalyze the case $Q \neq 0$.

\subsection{Lower bound to $\mu(t)$ in the discrete case.}
\label{appendix:mu_cases}

Note that the derivative in time of $\mu(\a(t),\b(t))$ is
\begin{equation}
\begin{split}
    \dot \mu \quad & = \quad
    -4 \left[\frac{1}{n}\sum x_i^2\right]^2 \left( \a^\top \b - \Phi \right) \ \a^\top \b.
\end{split}
\end{equation}
It thus decreases when $\a^\top \b>\Phi$ and when $\a^\top \b < 0$ and $\a^\top \b<\Phi$, it grows when $\a^\top \b > 0$ and $\a^\top \b<\Phi$.
This means that
\begin{itemize}[leftmargin=1em,itemsep=0.2em,parsep=0.1em]
    \item \textbf{Region A:} When $\a(0)^\top \b(0) > \Phi$, in Region A of Figure \ref{Fig:regions},
    we can bound 
    \[
    \mu(t) \geq \inf_{\a^\top \b>\Phi} \left[\frac{1}{n}\sum x_i^2\right] (\mednorm{\a}^2 + \mednorm{\b}^2) = 2 \Phi.
    \]
    Thus in this area we have that $2\Phi \leq \scale \leq \bar \scale$.
    
    \item \textbf{Region B:} When $\a(0)^\top \b(0) > 0$ and $\a(0)^\top \b(0)<\Phi$, in Region B of Figure \ref{Fig:regions}, 
    we can bound 
    \[
    \mu(t) \geq \mu(0) = \left[\frac{1}{n}\sum x_i^2\right] \big( \mednorm{\a(0)}^2 + \mednorm{\b(0)}^2\big).
    \]
    Thus in this area we have that $\scale(0) \leq \scale \leq 2\Phi \leq \bar \scale$, where $\scale(0)\geq Q_0 > 0$ is the norm of the first step in this area, when $Q \neq 0$.
\end{itemize}   
Note that this implies that our loss equipped with gradient descent is PLAT in Region A and Region B.
\begin{itemize}[leftmargin=1em,itemsep=0.2em,parsep=0.1em]
    \item \textbf{Region C:} When $\a(0)^\top \b(0) < 0$, in Region C of Figure \ref{Fig:regions}, the residuals decreases until $\a^\top \b = 0$. Thus the lowest point for $Q$ will be at the step $\tau$ that is the first step in which $\a^\top \b \geq 0$. This implies that the quantity $\scale$ will be at its minimum either at time $\tau$ or $\tau-1$
    \[
    \mu(t) \geq \min\{\mu(\tau-1), \mu(\tau)\}  \quad \text{where }\tau = \min_{t\in \N} \{ \a(t)^\top \b(t) > 0\}.
    \]
    In particular $\mu(t) \geq \min\{\mu(\tau-1), \mu(\tau)\} \geq Q(\tau_1)$, we need to show that when $Q\neq 0$ then $Q(\tau_1)\neq 0
    $.
    Thus in this area we will prove in the next section that we have that $Q(\tau_1) \leq \scale \leq 2\Phi \leq \bar \scale$.
\end{itemize}
This concludes the argument for all the cases except for $\a(0)^\top \b(0) < 0$,. We will now bound $\big|Q_i(\tau_1) \big|$ in terms of $\big|Q_i(0)\big|$, the learning rate $\eta>0$, and $\a(0)^\top \b(0)$.

\section{Lower bound on $\mu$ in Region C}
\label{Appendix:PLAT}
We prove in this section that
\begin{enumerate}
    \item The loss equipped with gradient descent is PL along the trajectories also in Region C.
    \item That GD escapes Region C very quickly, precisely see Proposition \ref{prop:final:PLAT}.
\end{enumerate}

This strategy achieves the goal of proving that—even when the dynamics stay a long time in Region C—the reduction in $Q$ is controlled and thus the convergence rate is linear.

Precisely, we prove in this section the following result.
\begin{proposition}
\label{prop:final:PLAT}
    Let at initialization $\a(0)^\top \b(0) < 0$ and $\Phi>0$. Let $\eta < \min\left\{\frac{1}{|\residual|}, \frac{2}{\bar \scale}\right\}$. t
    There exists $\tau$ such that $a(\tau)^\top b(\tau) > 0$,
    \begin{align*}
        2\sqrt{\eta}\Phi
        \quad < \quad
        Q(\tau)
        \quad = \quad
        O\left(\exp(-\eta)\right) \ \cdot \ Q(0)
        ,
    \end{align*}
    and
    \begin{align*}
        \tau
        \quad = \quad
        O\left(\frac{\log(\a(0)^\top \b(0))}{\eta \Phi} + \frac{|\a(0)^\top \b(0)|}{\eta \Phi} \right)
        .
    \end{align*}
\end{proposition}
We establish this this proposition in the next 3 subsections. Precisely, we establish some useful lemmas in the next subsection, then we show the statement for $1 > \eta |\epsilon| \geq 1/2$ in Section \ref{Appendix:PLAT_big_lr}.
Later in Section \ref{Appendix:PLAT_small_lr} we show that we can consider WLOG to be in a specific setting after a small number of steps. We proceed to show the actual speed of convergence.

\subsection{Preliminaries}
Note that
\begin{lemma}[Bound of $Q$ and $|\a^\top \b|$]
\label{lemma:scale_bound}
    Let $\a, \b \in \R^n$. If $n = 1$ $\scale^2 = Q^2 + 4\a^2\b^2$. If $n>1$ then
    \[
    \scale^2 
    \quad = \quad
    Q^2 + 4\mednorm{\a}^2\mednorm{\b}^2
    \quad \geq \quad
    Q^2 + 4|\a^\top\b|^2.
    \]
\end{lemma}
\begin{proof}[Proof of Lemma \ref{lemma:scale_bound}]
    Assume wlog $\a_i>\b_i$ for all $i$. Note that 
    \begin{equation}
    \begin{split}
        \scale^2 
        \quad &= \quad
        \left(\sum_{i}\a_i^2 + \sum_{i}\b_i^2 \right)^2
        \quad = \quad
        \sum_{i} \a_i^4 + \b_i^4 \ + \ 2\sum_{i}\a_i^2\b_i^2
        \ + \ 2\sum_{i \neq j}\a_i^2\b_j^2
        \\&= \quad
        \left(\sum_i Q_i\right)^2 \ + \ 4 \sum_{i} \sum_j \a_i^2\b_j^2
        \quad = \quad
        Q^2 + 4\mednorm{\a\b^\top}_{Frobenius}^2
        \quad = \quad
        Q^2 + 4\mednorm{\a}^2\mednorm{\b}^2
        \\&= \quad
        Q^2 + 4|\a^\top\b|^2  + \ 4 \sum_{i \neq j} \a_i^2\b_j^2.
    \end{split}
    \end{equation}
    This establishes Lemma \ref{lemma:scale_bound}.
\end{proof}

Also note that 
\begin{lemma}[Bounding the update]
\label{lemma:step_bound}
    Let $\a, \b \in $ Region C. Assume $\eta$ is such that one step of GD does not land in a different region. Then the quantity $C(k) := \frac{\scale}{|\a^\top \b|} \geq 2$ increases with one step of GD.
\end{lemma}
\begin{proof}[Proof of Lemma \ref{lemma:step_bound}]
Remind from Equation \eqref{eq:residuals-update} that
\begin{equation}
    \a(t+1)^\top \b(t+1)
    \quad = \quad
    \a(t)^\top \b(t) (1 + \eta^2 \residual^2) \underbrace{- \eta \residual \scale}_{positive}
\end{equation}
and 
\begin{equation}
    \scale(t+1)
    \quad = \quad
    \big(1 + \eta^2 \residual(t)^2 \big)\scale(t) \ \underbrace{- \ 4 \eta \residual(t)}_{negative}.
\end{equation}
The GD update on the quantity above is
% \begin{equation}
% \begin{split}
%     &\frac{\residual \big(1 - \eta \scale(t) + \eta^2 \residual(t) \a(t)^\top \b(t) \big) \cdot \big[ \scale (1 + \eta^2 \residual^2) \overbrace{- 4\eta \residual \a^\top \b}^{negative}
%     \big]
%     }{\a^\top \b (1 + \eta^2 \residual^2) \underbrace{- \eta \residual \scale}_{positive}}
%     \\ & \leq \ 
%     \frac{\residual \scale
%     }{\a^\top \b}
%     \cdot 
%     \underbrace{\frac{(1 + \eta^2 \residual^2)\big(1 - \eta \scale(t) + \eta^2 \residual(t) \a(t)^\top \b(t) \big) 
%     }{1 + \eta^2 \residual^2}}_{\text{smaller than }1}
%     -
%     \underbrace{\frac{4\eta \residual^2 \a^\top \b
%     }{\a^\top \b (1 + \eta^2 \residual^2)} }_{\text{negative}}
% \end{split}
% \end{equation}
\begin{equation}
\begin{split}
    &-\frac{\scale (1 + \eta^2 \residual^2) - 4\eta \residual \a^\top \b
    }{\a^\top \b (1 + \eta^2 \residual^2) - \eta \residual \scale}
    \quad \geq \quad 
    -\frac{\scale
    }{\a^\top \b}
\end{split}
\end{equation}
if and only if
\begin{equation}
\begin{split}
    &|\a^\top \b| \scale (1 + \eta^2 \residual^2) - 4\eta \residual |\a^\top \b|^2
    \quad \leq \quad 
    \scale \big( |\a^\top \b| (1 + \eta^2 \residual^2) - \eta \residual \scale \big)
\end{split}
\end{equation}
if and only if
\begin{equation}
\begin{split}
    &- 4\eta \residual |\a^\top \b|^2
    \quad \leq \quad 
    - \eta \residual \scale^2 
\end{split}
\end{equation}
if and only if
\begin{equation}
\begin{split}
    &4|\a^\top \b|^2
    \quad \leq \quad 
    \scale^2 
\end{split}
\end{equation}
which holds by Cauchy-Schwartz.

\end{proof}

% Note thus that we can rewrite the step on the quantity $\a^\top \b$ 
% \begin{equation}
% \begin{split}
%     \a(t+1)^\top \b(t+1)
%     \quad &= \quad
%     \a(t)^\top \b(t) \ - \ \eta \residual(t) \scale(t)
%     \ + \  \eta^2 \residual(t)^2 \a(t)^\top \b(t)
%     % \\&\geq \quad
%     \quad = \quad
%     \left( 1 - \eta C(t) + \eta^2 \residual(t)^2 \right)
%     \a(t)^\top \b(t)
%     .
% \end{split}
% \end{equation}

\subsection{Bigger Step Size}
\label{Appendix:PLAT_big_lr}
Note that WLOG we consider small steps, because if the step is big we either fall in another area or in a place of region C where the step size is much smaller with respect to $\scale$ than at the step before. In particular, when 
\[
\eta \quad = \quad
\frac{\alpha}{|\residual|}
\quad \geq \quad
\frac{C(k)\alpha}{|\residual|}\frac{|\a^\top \b|}{\mednorm{\a}^2 + \mednorm{\b}^2}
\quad \geq \quad
\frac{2\alpha}{|\residual|}\frac{|\a^\top \b|}{\mednorm{\a}^2 + \mednorm{\b}^2},
\]
with $\alpha < 1$. Then the update on $\a^\top \b$ (which is negative) is 
\begin{equation}
\label{eq:big_step_size}
\begin{split}
    \a(t+1)^\top \b(t+1)
    \quad &= \quad
    \a(t)^\top \b(t) \ - \ \eta \residual(t) \scale(t)
    \ + \  \eta^2 \residual(t)^2 \a(t)^\top \b(t)
    \\&\geq \quad
    \left[ \left( 1 - \eta \residual(t) \right)^2 - (C(k)-2) \alpha \right] \cdot
    \a(t)^\top \b(t)
    \\&\geq \quad
    \left[ \left( 1 - \alpha \right)^2 - (C(k)-2) \alpha \right] \cdot
    \a(t)^\top \b(t).
\end{split}
\end{equation}
Here, if algorithm jumped in Region A or B where $\a^\top \b>0$ after one step, when $\scale(t)$ is much bigger than $|\a^\top \b|$ we are done. Otherwise, when $\alpha > 1/2$ or we jumped at least at one quarter of the distance between the initial $\a^\top \b$ and the \emph{positive} target $\Phi$. This implies that $|\residual(t)|$ decreased to at most $\Phi + \frac{\a(t)^\top \b(t)}{\alpha^2}$ from $\Phi + \a(t)^\top \b(t)$. This means that we are still in the assumptions above. Within a $k$ steps we either jumped on the other side or we reached a place where $\a(k)^\top \b(k)< 2^{-k} \a(0)^\top \b(0)$ and $Q(k) > (1 - \alpha^2)^{k} Q(0)$ which decays exponentially but slower. Since $\scale(k)^2 > Q(k)^2 + 4(\a(k)^\top \b(k))^2$ there exists a moment where $\alpha \leq 1/2$ or we reach a place where $Q(k) > 2\sqrt{3}|\a(k)^\top \b(k)|$ and thus $\scale(k) > 4 |\a(k)^\top \b(k)|$. This happens in constant time as $\alpha$ is bounded from above and below. In this regime, 
\begin{equation}
\begin{split}
    \a(t+1)^\top \b(t+1)
    \quad &= \quad
    \a(t)^\top \b(t) \ - \ \eta \residual(t) \scale(t)
    \ + \  \eta^2 \residual(t)^2 \a(t)^\top \b(t)
    \\&\geq \quad
    \a(t)^\top \b(t) \ - \ 4 \eta \residual(t) \a(t)^\top \b(t)
    \ + \  \eta^2 \residual(t)^2 \a(t)^\top \b(t)
    \\&\geq \quad
    \underbrace{\left[ \left( 1 - \alpha \right)^2 -2\alpha \right]}_{negative}
    \a(t)^\top \b(t).
\end{split}
\end{equation}
This implies that we are in Region A or B and $\min_k \scale(k) = Q(k_1) > (1 - \alpha^2)^{k_1} Q(0)$ for some $k_1 \leq k$.

\subsection{Small Step Sizes}
\label{Appendix:PLAT_small_lr}
The difficult case to deal with analytically is the one where the dynamics stays in Region C for long.

% This strategy achieves the goal of proving that—even when the dynamics stay a long time in Region C—the reduction in $Q$ is controlled and thus the convergence rate is linear.
% % The proof goes essentially as follow:

% \begin{itemize}[leftmargin=1em,itemsep=0.2em,parsep=0.1em]
%     \item \textbf{Step 1:} We define the auxiliary recurrences for \(z_k\) and \(w_k\) that bound \(\varepsilon(t)\) and \(Q(t)\) respectively. 
%     \item \textbf{Step 2:} We show that the imbalance \(w_k\) decays at most by a multiplicative factor \(\exp(-2\eta^2 \sum z_k^2)\) and then bound the sum \(\sum z_k^2 \le C_1\,\tau\).  
%     \item \textbf{Step 3:} We derive an upper bound on the escape time \(\tau\) by comparing the multiplicative decay of \(z_k\) to the threshold \(-\Phi\). 
%     \item \textbf{Step 4:} Combining these estimates shows that the loss in the imbalance \(Q(\tau)\) is bounded by a factor depending on \(\log(|\varepsilon(0)|/\Phi)\) and \(Q(0)\).  
% \end{itemize}

We compute here a lower bound on $|Q_i(\tau)|$.
The idea here is that the residuals $\a^\top \b - \Phi$ will converge as $\exp(-\eta t)$ and the quantity $Q_i(t)$ at most as $\exp(-\eta^2 t)$, thus $\a(t)^\top \b(t)$ crosses $0$ before $|Q_i(t)|$ gets too small.

Note that enforcing Cauchy-Schwartz and the fact that $1 \leq \frac{1}{2|\residual(t)|}$, we establish that at every step of gradient descent we have the following updates on the following quantities
\begin{equation}
\begin{split}
\label{eq:iteration:1}
\residual(t+1)
\quad &= \quad
\left(1 - \eta  \scale(t) \right) \residual(t) 
\ + \ \eta^2 \residual(t)^2 \a(t)^\top \b(t)
\\ &\geq \quad 
\left(1 - \eta \scale(t) + \eta^2 \residual(t) \frac{\scale(t)}{2} \right) \residual(t) 
\\ &\geq \quad 
\left(1 - \frac{3}{4} \eta \scale(t)\right) \residual(t),
\end{split}
\end{equation}
Note that the $\geq$ are there because we are working with negative quantities.
Moreover, arguably the multiplying constant $C_1$ for which we have an equality in the Cauchy-Schwartz inequality: $\mednorm{\a(k)}^2 + \mednorm{\b(k)}^2 = C_1(k) |\a(k)^\top \b(k)|$ is increasing with $k$ until $\tau$. This implies that closer to the boundary with Region B, we can have a much better bound than this one. Analogously, remind from Equation \eqref{eq:Qi-update} that
\begin{equation}
\begin{split}
Q_i(t+1)
\quad &= \quad
\left(1 - \eta^2 \residual(t)^2 \right) Q_i(t).
\end{split}
\end{equation}

\paragraph{Bounding Sequences.}
We define here two coupled sequences which serve as bounds to the evolution of $\residual$ and $Q$ along then trajectory. We study their behavior and we infer bounds on the behavior of our system.
\begin{definition}
\label{def:sequence}
    Let $\a(0), \b(0) \in \R^n$. 
    Let $\eta < \min\left\{\frac{1}{2|\residual|}, \frac{2}{\bar \scale}\right\}$. Define the sequence $\{z_k, w_k\}_{k=0}^\infty$ such that $z_0 = \residual(0) < - \Phi$, $w_0 = Q_0 > 0$, and for all $k\in \N$ we have
    \begin{equation}
    \label{eq:def:sequence}
    \begin{split}
    M_{k} \ &= \ \max \Big\{ w_k, -2z_k - 2\Phi\Big\} 
    \\[0.1cm]
    z_{k+1} &= \ \left( 1 - \frac{3\eta}{4}  M_k  \right) z_k
    \\[0.1cm]
    w_{k+1} &= \ \big( 1 - \eta^2 z_k^2 \big) w_k.
    \end{split}
    \end{equation}
    Define $\tau_1:= \min_{k \in \N}\{z_k > -\Phi \}$.
\end{definition}
Note that we have
\begin{lemma}[Bounding with the sequences]
\label{lemma:bound}
For all $1 \leq k < \tau_1$ such that $z_k < 0$ we have 
\[z_k \leq \residual_k
\qandq
w_k < 
% \mednorm{a(k)}^2 - \mednorm{b(k)}^2 < 
Q_k.
\]
Moreover, $w_k,-z_k \geq 0$ are strongly monotone decreasing for $k < \tau_1$ and $\a(\tau_1)^\top \b(\tau_1)>0$, thus for all $k \leq \tau_1$ we have $\eta < 2/\max\{-2z_{k}, w_{k}\}$. 
\end{lemma}
\begin{proof}
    Note that this is the case for $k=0$. As for the inductive step, Eq.\ \eqref{eq:iteration:1}, Cauchy-Schwartz inequality, and Eq.\ \eqref{eq:def:sequence} establish the first point. Note that $z_{\tau_1-1} + \Phi < \a(\tau_1-1)\b(\tau_1-1) < 0$, then the first point and the definition of $\tau$ imply that $0 < z_{\tau_1} + \Phi < \a(\tau_1)^\top \b(\tau_1)$. Note that after the first step $w_1<w_0$ and $z_1>z_0$ since Cauchy-Schwartz implies that $\eta < 2/\max\{-2z_0, w_0\}$. Inductively, for all $i$ we have $\eta < 2/\max\{-2z_i, w_i\}$, thus fact that $z_i < 0$ for all $k < \tau_1$ implies that $w_{i+1}<w_i$ is strongly monotonically decreasing, that $z_{i+1}>z_i$ is strongly monotonically increasing, and that $\eta < 2/\max\{-2z_{i+1}, w_{i+1}\}$.
\end{proof}
As explained before, for all $t$ we have $\mu(t) \geq \max\{ \mu(\tau_1), \mu(\tau_1-1)\}$ and $\mu(t)\geq \sum_i |\a_i^2(t) - \b_i^2(t)| \geq w_t \geq w_{t+1}$. Thus for all $t$ we have $\mu(t) \geq w_{\tau_1}$. This and the lemma above show that
\begin{lemma}
    We have that
    $\a(\tau)^\top \b(\tau) > -\Phi$ and for all $t \in \N$ we have $\mu(t) \geq w_{\tau}$.
\end{lemma}

\paragraph{Behavior of the sequence: Case 1.}
We assume in this paragraph that 
$w_0 \geq -2z_0 - 2\Phi$.
We characterize $\tau$ and $Q(\tau)$ in this case.

Note that we can assume that $w_0 \geq -2z_0 - 2\Phi$, indeed note that
\begin{lemma}
\label{lemma:setting_of_sequence}
    Assume we are in Region C and $-2 \a^\top \b \geq Q$. Note that within a finite time we have $Q > -2 \a^\top \b$. Precisely, the number of steps needed is $s_1 + s_2$ where 
    \begin{equation}
    \begin{split}
        s_1 \ &= \ 
        C \frac{\log\big(\max\{|\a(0)^\top \b(0)|, \sqrt{2}\Phi\} \big) - \log(\sqrt{2}\Phi)}{\eta \sqrt{2}\Phi}
        \\ s_2 \ &= \
        C \frac{\log\big(|\a(s_1)^\top \b(s_1)|\big) - \log(Q(s_1))}{\eta \Phi}
        .
    \end{split}
    \end{equation}
    where $C$ is an absolute constant.
\end{lemma}
\begin{proof}[Proof of Lemma \ref{lemma:setting_of_sequence}]
Let us define $\alpha(k)$ as above 
\[
\alpha(k) = \eta |\residual(k)|.
\]
Note that for all $k < \tau_1$ we have $\alpha(k+1) < \alpha(k)$ as $|\residual(k)|$ is strictly decreasing. However, note that we have the following bounds $\alpha(\tau_1-1) \geq \eta \Phi \geq \alpha(\tau)$.
This and Equation \eqref{eq:big_step_size} implies that
\begin{equation}
\a(k)^\top \b(k)
\quad = \quad
(1 - \eta \Phi)^k \a(0)^\top \b(0).
\end{equation}
Note that for all $\Phi$ such that $\eta \leq 1/2\Phi$ there exists a constant $c \leq (\sqrt{2}-1) \Phi$ such that $\eta(\Phi+c)^2\leq \Phi$. Note that when $|\epsilon(k)| > \sqrt{2}\Phi$ then a number of steps that is big-O of $\log\big(\max\{|\a(0)^\top \b(0)|, \sqrt{2}\Phi\} \big) - \log(\sqrt{2}\Phi)$ divided by $\sqrt{2}\Phi \eta$ we have $\a(0)^\top \b(0) \geq (\sqrt{2}-1) \Phi$.
From here on, $Q$ decreases slower than $|\a(0)^\top \b(0)|$ here again, after a number of steps which is big-O of $1/\eta$ multiplied by the log ratio of the two quantities steps, $|\a(k)^\top \b(k)|$ becomes $1/2$ of $Q$ in size\footnote{In reality this is much faster, this is a construction of the proof. However, it is tight enough to conclude with a linear rate.}. The ratio of their rates asymptotically goes as $\frac{1 - \eta \Phi}{1 - \eta^2 \Phi^2}$.
\end{proof}

\begin{lemma}[Rate of convergence 1 - Sequence.]
    \label{lemma:rate:1}
    If $w_0 \geq -2z_0 - 2\Phi$, define $c_1 := \frac{w_0}{2} - \frac{\sqrt{w_0 (w_0 - 4\eta z_0^2)}}{2} > 0$, then
    \begin{equation}
        c_1
        \quad < \quad 
        w_{\tau_1}
        \quad < \quad
        w_0 - \eta^2 \left( \Phi \right)^2 \frac{\big(\a(0)^\top \b(0)\big)^2}{w_0}
    \end{equation}
    and
    \begin{equation}
        \frac{1}{\eta (w_0)^{3/2}} 
        \quad < \quad 
        \tau_1
        \quad < \quad 
        \frac{\a(0)^\top \b(0)}{\eta c_1^{3/2}} \left(\Phi\right)^{-1} 
        + 1
        .
    \end{equation}
\end{lemma}

\begin{proof}
    Note that for all $k < \tau_1$ we have
    \begin{equation}
    \label{eq:lemma:rate:1}
        \frac{(z_{k+1} - z_k)^2}{w_{k+1} - w_k}
        \ = \ 
        \frac{\eta^2 w_k^2 z_k^2}{-\eta^2 z_k^2 w_k}
        % \ > \ - 2 \eta w_k z_k^2 + \eta^2 z_k^2 w_k + 2\eta w_k z_k^2 
        \ = \ 
        -w_k.
    \end{equation}
    Note that $z_{\tau_1-1}-z_0 < 
    |\a(0)^\top \b(0)| \leq z_{\tau_1}-z_0$. We thus obtain that 
    \begin{equation}
    \begin{split}
    % \label{eq:lemma:rate:2}
        \a(0)^\top \b(0)
        \ &\sim \
        z_0-z_{\tau_1}
        \ = \
        \sum_{k=0}^{\tau_1-1}
        z_k-z_{k+1}
        \ = \
        \sum_{k=0}^{\tau_1-1} \sqrt{w_k-w_{k+1}} \cdot w_k
        \ = \
        \eta \sum_{k=0}^{\tau_1-1} z_k (w_k)^{3/2}
    \end{split}
    \end{equation}
    This implies that
    \begin{equation}
        \eta (\tau_1-1) \Phi (w_{\tau-2})^{3/2}
        \ < \
        \a(0)^\top \b(0)
        \ \leq \
        \eta \tau_1z_0 (w_0)^{3/2}.
    \end{equation}
    Next we proceed bounding $w_{\tau_1}$ so that we can bound $\tau_1$.
    Note that the fact that $z_k < - \Phi < 0$ for all $k < \tau_1$ and Sedrakyan's lemma imply that
    \begin{equation}
    \begin{split}
        w_0-w_{\tau_1} &= \sum_{i=0}^{\tau_1-1} w_k-w_{k+1} = \sum_{k=0}^{\tau_1-1} \frac{(z_{k+1}-z_k)^2}{w_k} > \frac{(z_{\tau_1} - z_0)^2}{\sum_{i=0}^{\tau_1-1} w_k} 
        \\&> \frac{\big(\a(0)^\top \b(0)\big)^2}{\sum_{i=0}^{\tau_1-1} (1 - \eta^2 z_{\tau_1-1}^2) w_0}
        = \eta^2 z_{\tau_1}^2\frac{\big(\a(0)^\top \b(0)\big)^2}{w_0}
        \\&>
        \frac{\a(0)^\top \b(0)}{\tau_1w_0}.
    \end{split}
    \end{equation}
    And this implies that
    \begin{equation}
        w_{\tau_1}= w_0 + (w_{\tau_1}- w_0) < w_0 - \eta^2 \left( \Phi \right)^2 \frac{\big(\a(0)^\top \b(0)\big)^2}{w_0}
    \end{equation}
    Moreover, we have
    \begin{equation}
    \begin{split}
        w_0-w_{\tau_1} &= \sum_{k=0}^{\tau_1-1} w_k-w_{k+1} =
        \eta^2 \sum_{i=0}^{\tau_1-1} z_k^2 w_k 
        < \eta^2 \sum_{i=0}^{\tau_1-1} z_k^2 (1 - \eta^2 z_{\tau_1}^2)^k w_0 
        \\&< \eta^2 z_0^2 w_0 \sum_{i=0}^{\tau_1-1} (1 - \eta w_{\tau_1})^k (1 - \eta^2 z_{\tau_1}^2)^k
        < \eta z_0^2 w_0 \frac{1}{ w_{\tau_1}+ \eta z_{\tau_1}^2 - \eta^2 w_{\tau_1}z_{\tau_1}^2}
        \\&
        < \eta \frac{z_0^2 w_0}{w_{\tau_1}}
        % \\&
        % < \eta^2 z_0^2 \sum_{k=0}^{\tau_1-1} w_k
        % < \eta^2 z_0^2 \sum_{k=0}^{\tau_1-1} (1 - \eta^2 z_{\tau_1}^2)^k w_0
        % =
        % \frac{z_0^2}{z_{\tau_1}^2}w_0
        .
    \end{split}
    \end{equation}
    This implies with $0 < w_{\tau_1}< w_0$ that $w_{\tau_1}(w_0-w_{\tau_1}) < \eta z_0^2 w_0$, thus
    \begin{equation}
        w_{\tau_1}^2 - w_0 w_{\tau_1}+ \eta z_0^2 w_0 > 0.
    \end{equation}
    Note that $w_0 > -2z_0$ and $\eta < -2/z_0 $ implies that $w_0 \geq \eta z_0^2$ then, solving, we obtain
    \begin{equation}
        0 < c_1 := \frac{w_0}{2} - \frac{\sqrt{w_0 (w_0 - 4\eta z_0^2)}}{2} 
        < w_{\tau_1}
        < w_0 - \eta^2 \left( \Phi \right)^2 \frac{\big(\a(0)^\top \b(0)\big)^2}{w_0}
    \end{equation}
    $w_{\tau_1}> \frac{w_0}{2} - \frac{\sqrt{w_0 (w_0 - 4\eta z_0^2)}}{2}$
    Thus, opportunely bounding $w_{\tau-2}$ we obtain
    \begin{equation}
        \eta (\tau_1-1) \Phi c_1^{3/2}
        \ < \
        \a(0)^\top \b(0).
    \end{equation}
    That we can reorganize as
    \begin{equation}
        \tau_1
        \ < \
        \frac{\a(0)^\top \b(0)}{\eta c_1^{3/2}} \left(\Phi\right)^{-1}
        +1
        .
    \end{equation}
    % This implies
    % \begin{equation}
    %     3\tau/2\log(\tau) \log(1 - \eta^2 z_0^2)
    %     \ < \
    %     \frac{\a(0)^\top \b(0)}{\eta w_0^{3/2}} \left(\Phi\right)^{-1}
    %     +1
    %     .
    % \end{equation}
    % This implies
    % \begin{equation}
    %     \tau_1\log(\tau) \left( - \eta^2 z_0^2 \right)
    %     \ < \
    %     \frac{\a(0)^\top \b(0)}{\eta w_0^{3/2}} \left(\Phi\right)^{-1}
    %     +1
    %     .
    % \end{equation}
    % $-\eta$ is negative so this is fucked up. We can do the following: It happens for $\tau$ such that $w_{\tau_1}= c w_0 > 0$. Thus
    % \begin{equation}
    %     \tau
    %     \ < \
    %     \frac{\a(0)^\top \b(0)}{\eta w_0^{3/2} c^{3/2}} \left(\Phi\right)^{-1}
    %     +1
    %     .
    % \end{equation}
    % Thus $c$ satisfies
    % \begin{equation}
    %      \Big(1 - \eta^2 \Phi \Big)^\tau_1> c = \prod (1-\eta^2 z_k^2)^\tau_1> \Big(1 - \eta^2 z_0 \Big)^\tau, 
    % \end{equation}
    % Thus
    % \begin{equation}
    %      \log(c) >  - \eta z_0 w_0^{-3/2}, 
    % \end{equation}
    % Then actually
    % \begin{equation}
    %      \tau_1< \frac{\a(0)^\top \b(0)}{\eta w_0^{3/2} \exp( -\eta z_0 w_0^{-3/2} ) } \left(\Phi\right)^{-1}
    %     +1
    % \end{equation}

    % \begin{equation}
    %     \tau_1\left( w_0 - \frac{\a(0)^\top \b(0)}{\tau_1w_0} \right)^{3/2}
    %     \quad < \quad
    %     \frac{\a(0)^\top \b(0)}{\eta} \left(\Phi\right)^{-1} 
    %     + w
    %     .
    % \end{equation}
    Thus
    \begin{equation}
        \frac{1}{\eta (w_0)^{3/2}} 
        \quad \leq \quad
        \tau_1
        \quad \leq \quad
        \frac{\a(0)^\top \b(0)}{\eta c_1^{3/2}} \left(\Phi\right)^{-1} 
        + 1
        .
    \end{equation}
\end{proof}

\begin{lemma}[Rate of convergence 1.]
    \label{lemma:conv:1}
    If $\sum_i |\a_i^2(0) - \b_i^2(0)| \geq -2 \a(0)^\top \b(0) > 0$, define $c_1 := \frac{w_0}{2} - \frac{\sqrt{w_0 (w_0 - 4\eta z_0^2)}}{2} > 2\sqrt{\eta}\Phi$ as above, then
    \begin{equation}
        c_1
        \quad < \quad 
        \sum_i |a_i^2(\tau_1) - b_i^2(\tau_1)|
        \quad < \quad
        \sum_i |\a_i^2(0) - \b_i^2(0)| - \eta^2 \left( \Phi \right)^2 \frac{\big(\a(0)^\top \b(0)\big)^2}{\sum_i |\a_i^2(0) - \b_i^2(0)|},
    \end{equation}
    \begin{equation}
        \a(0)^\top \b(0)
        \quad > \quad 
        0,
    \end{equation}
    and
    \begin{equation}
        \frac{1}{\eta \left( \sum_i |\a_i^2(0) - \b_i^2(0)| \right) ^{3/2}} 
        \quad < \quad 
        \tau_1
        \quad < \quad 
        \frac{\a(0)^\top \b(0)}{\eta c_1^{3/2}} \left(\Phi\right)^{-1} 
        + 1
        .
    \end{equation}
\end{lemma}

\begin{proof}
    One bound comes from Lemma \ref{lemma:rate:1} and Lemma \ref{lemma:bound}. The other one comes by just following the proof of Lemma \ref{lemma:rate:1}.
\end{proof}

This concludes the proof of Proposition \ref{prop:final:PLAT} and shows that the $Q_i$ are lower bounded for all $i$ when initialization is in Region C and $\eta < \min\left\{\frac{1}{|\residual|}, \frac{2}{\bar \scale}\right\}$ .

% % Recall that our dataset is $\mathcal{D} = \{(x_i,y_i)\}_{i=1}^m \subset \R \times \R$, note that the empirical second moment of the inputs $\E_{\mathcal{D}}[x^2]>0$, otherwise the subset has at most one data point and it is 0. We will fit it with a linear network and MSE. Notice that linear networks represent linear functions, so the best it can do is the best linear fit and we can rewrite the loss depending on that:
% % \begin{equation}
% % \label{eq:loss}
% %     L(\a,\b) 
% %     \quad = \quad 
% %     \left[\frac{1}{n}\sum x_i^2\right] \bigg( \a^\top \b - \underbrace{\Phi}_{=:\Phi} \bigg)^2 
% %     \ + \ 
% %     \underbrace{\sum y_i^2 - \left[\sum y_ix_i\right]^2 \left[\frac{1}{n}\sum x_i^2\right]^{-1}}_{=: \mathcal{R}^2 \geq 0}.
% % \end{equation}
% % We denote by $\mathcal{R}^2 := \sum y_i^2 - \left[\sum y_ix_i\right]^2 \left[\frac{1}{n}\sum x_i^2\right]^{-1}$ the minimum attainable value of the loss, and we denote by $\Phi:=\Phi$ the best possible linear estimator, which is the value attained by the linear network at the end of training. Essentially, the interesting part of the loss is the fist addendum above, that is the part we can optimize.

\section{Convergence Speed Case by Case}
\label{app:speed}
This section serves as merger for all the theory made before. Precisely, here we use the analysis developed to prove Theorem \ref{theo:speed}.

We prove below and in Appendix \ref{Appendix:PLAT} that in the three different regions of the landscape we have different PL constants $\mu$ for $\residual$-and then for $L$.
Precisely, if $\residual \geq 0$ then $\mu > 2\Phi$, if $\residual < -\Phi/2$ then $\mu = Q(\tau)$, and if $-\Phi/2 \leq \residual < 0$ then $\mu = \Phi$.
This implies that we have convergence with the minimum of $Q(\tau)$ and $2\Phi$ as PL constant until $|\residual| > \Phi/2$, then we have convergence with $\Phi$ as PL constant from then on.

\subsection{The Slow Case}
We start by dealing with the case in which convergence is very slow. Imagine during the training $|\residual(k)| \ll 1$ and $\eta \geq 2/\scale(k)$.
In this case, convergence, if it happens, happens only at most logarithmically fast at least for a first phase, precisely in the best case with $\eta = 2/\scale$ we have
\begin{equation}
\begin{split}
    |\residual(k+1)|
    &= \left| (1 - \eta \scale) \residual(k) + \eta^2 \residual^2(k) (\residual(k) + \Phi) \right|
    \\&\leq
    \left|-\residual(k)
    + \frac{4}{\scale^2} \residual(k)^2 (\residual(k) + \Phi)
    \right|
    \\&\leq
    \left(1 - \frac{4}{\scale^2} \Phi \residual(k)\right)\ |\residual(k)|.
\end{split}
\end{equation}
After two steps thus, we have approximately to the second order in $\residual$
\begin{equation}
\begin{split}
    |\residual(k+2)|
    & \sim
    \left(1 - \frac{8}{\scale^2} \Phi \residual(k) \right)\ |\residual(k)| \ + \ O(\residual(k)^3).
\end{split}
\end{equation}
Analogously the norm changes to little
\begin{equation}
\begin{split}
    \scale(k+2)
    &= 
    \scale(k) - 8 \residual(k) (\residual(k) + \Phi) + 8 \residual(k+2) (\residual(k+2) + \Phi) + 4\residual(k)^2 - 4\residual(k+2)^2 
    \\& \quad - \ \eta^2 (1 - \eta^2 \residual(k)^2) Q(k)^2 \big( \residual(k)^2 + \residual(k+1)^2 (1 - \eta^2 \residual(k+1)^2) \big)
    \\& = 
    \scale(k) - O \big( \eta^3 \residual(k)^3 + \eta^2 \residual(k)^2 \big).
\end{split}
\end{equation}
Thus the situation does not change for the next 2 steps and this establishes the last comment of Theorem \ref{theo:speed}.

\subsection{Positive residuals}
First note that $\eta \scale \residual > \eta^2 \residual^2 (\residual + \Phi)$, indeed $\scale > 2 (\residual + \Phi)$ by Cauchy Schwartz and $\eta \leq \frac{\sqrt{2}}{\residual} < \frac{2}{\residual}$. This implies that when $\eta$ is small or infinitesimal, the gain is at least 
\begin{equation}
\begin{split}
    \residual(k+1)
    &= (1 - \eta \scale) \residual(k) + \eta^2 \residual (\residual + \Phi)
    \\&\leq
    \residual(k) - \eta \scale \residual(k) \left(1 - \frac{\eta}{2} \residual(k) \right)
    \\&\leq
    \left(1 - \frac{2-\sqrt{2}}{2} \eta \scale \right)\residual(k).
\end{split}
\end{equation}
Next note that $\frac{x}{\sqrt{x^2 + y^2}} = \sqrt{1 - \frac{y^2}{x^2+y^2}} \leq 1 - \frac{1}{2}\frac{y^2}{x^2+y^2}$.
When the step size is big, instead, $\eta \sim \frac{2}{\bar \scale}(1-\delta)$, $\delta>0$ we have that
\begin{equation}
\begin{split}
    |\residual(k+1)|
    &= \left| (1 - \eta \scale) \residual(k) + \eta^2 \residual^2(k) (\residual(k) + \Phi) \right|
    \\&\leq
    \left|\residual(k) - \frac{2 \scale (1 - \delta)}{\sqrt{\scale^2 + 4\Phi^2}} \residual(k)
    + \frac{4(1-\delta)^2}{\scale^2 + 4\Phi^2} \residual(k)^2 (\residual(k) + \Phi)
    \right|
    \\&\leq
    \left|(-1+2\delta)\residual(k) + \frac{4\Phi^2(1-\delta)}{\scale^2 + 4\Phi^2} \residual(k)
    + \frac{4(1-\delta)}{\scale^2 + 4\Phi^2} \residual(k)^2 (\residual(k) + \Phi)
    \right|
    \\&\leq
    (1 - 2\delta)\residual(k).
\end{split}
\end{equation}
This implies that within our learning rate boundaries we have exponential convergence with rate either controlled by $\eta$ or $\delta$ at power 1. 

In case $\frac{2}{\bar \scale} \leq \eta \leq \frac{2}{\scale(0)}(1 - \delta)$ then convergence happens exponentially but in time $O\big(\eta^{-2}\big)$. For instance
$\eta \sim \frac{2}{\bar \scale}$ we have that
\begin{equation}
\begin{split}
    |\residual(k+1)|
    &= \left| (1 - \eta \scale) \residual(k) + \eta^2 \residual^2(k) (\residual(k) + \Phi) \right|
    \\&\leq
    \left|\residual(k) - \frac{2 \scale}{\sqrt{\scale^2 + 4\Phi^2}} \residual(k)
    + \frac{4}{\scale^2 + 4\Phi^2} \residual(k)^2 (\residual(k) + \Phi)
    \right|
    \\&\leq
    \left|-\residual(k) + \frac{4\Phi^2}{\scale^2 + 4\Phi^2} \residual(k)
    + \frac{4}{\scale^2 + 4\Phi^2} \residual(k)^2 (\residual(k) + \Phi)
    \right|
    \\&\leq
    (1 - \eta^2 \Phi^2)\residual(k).
\end{split}
\end{equation}

\subsection{Negative residuals $-\Phi/2 < \residual < 0$}
When the residuals are small negative we have exponential convergence, precisely, for very small $\eta \ll 1$ we have rate at least $(1 - \eta \Phi)$:
\begin{equation}
\begin{split}
    |\residual(k+1)|
    &= \left| (1 - \eta \scale) \residual(k) + \eta^2 \residual^2(k) (\residual(k) + \Phi) \right|
    \\&\leq
    \left| (1 - \eta \scale) 
    + \frac{\eta^2 }{4} \Phi^2
    \right| |\residual(k)|
    \\&\leq
    \left(1 - \eta \Phi \right)\ |\residual(k)|.
\end{split}
\end{equation}
For bigger $\eta = \frac{2}{\bar \scale}$, we have convergence with rate about $\sim 2$. The maximum over $\scale(0), \residual(k)$ in the region in which $\residual(k)= -c\frac{\Phi}{2}$ with $c \in (0,1]$
\begin{equation}
\begin{split}
    \max|\residual(k+1)|
    &= \max|(1 - \eta \scale) \residual(k) + \eta^2 \residual^2(k) (\residual(k) + \Phi)|
    \\&\leq
    \max\left|\residual(k) - \frac{2 \scale}{\sqrt{\scale(0)^2 + 4\Phi^2}} \residual(k)
    + \frac{4}{\scale(0)^2 + 4\Phi^2} \residual(k)^2 (\residual(k) + \Phi)
    \right|.
\end{split}
\end{equation}
Note that the minimum in $\scale(0)$ of this last equation is for $\sqrt{\scale(0) + 4\Phi^2} = \delta + 2\Phi$ for some $\delta > 0$ which satisfies $\delta \ll 1$. This is independent of the size of $\residual$. Along this trajectory, $\a^\top \b = (1-c/2)\Phi$ and $\scale \leq (2-c)\Phi + \delta$ This implies
\begin{equation}
\begin{split}
    \max_{\scale(0), \residual(k)}|\residual(k+1)|
    \ &\leq
    \max_{\residual(k)}\left|\residual(k) - \frac{2(2-c)\Phi}{\delta + 2\Phi} \residual(k)
    + \frac{4}{(\delta+2\Phi)^2} \residual(k)^2 (\residual(k) + \Phi)
    \right|
    \\&\leq
    \max_{\residual(k)}\left|-\frac{(2 - 2c)\Phi + \delta}{\delta + 2\Phi}
    + \frac{c(2-c)}{(\delta+2\Phi)^2} \Phi^2
    \right| |\residual(k)|
    \\&\leq
    \left|-1 + c\frac{2\Phi}{\delta + 2\Phi}
    + c(2-c)\frac{\Phi^2}{(\delta+2\Phi)^2}
    \right| |\residual(k)|
    \ \leq \ \left|c-1 -\frac{c^2}{4}+\frac{c}{2}
    \right| |\residual(k)|
    \\&\underset{c=1}{\leq}
    % \left|
    % \frac{1}{4} - 
    % \frac{3\delta}{2}\frac{1}{\delta + 2\Phi}
    % \right| |\residual(k)|
    % \ \leq \
    \frac{1}{4}|\residual(k)| \ = \ \frac{1}{8}\Phi.
\end{split}
\end{equation}
The maximum of $|c^2/4 -3c/2 + 1|$ over $c \in (0, 1]$ is $c=1$.

In the case of $c=1$, on the next step, in this case, we are in the positive residuals setting with $\scale$ as follows
$2 \cdot \a^\top \b  = \frac{9}{4}\Phi + \delta$.
Here, then
\begin{equation}
\begin{split}
    |\residual(k+2)| 
    \ &\leq\ 
    \left| - \frac{5}{4} + \frac{1}{4\Phi^2} \residual(k+1) (\residual(k+1)+\Phi)\right| \residual(k+1) 
    \\&\leq\
    \frac{1}{16}\left(5 - \frac{1}{16}\right) |\residual(k)|. 
\end{split}
\end{equation}
So after 2 steps, we had a linear shrink of $5/16$ and the linear convergence with constant $\mu=\Phi$ restarts, this is the plus 2 of the theorem.

\subsection{Negative residuals $\residual \leq \Phi/2$} 
This case is taken care of in Appendix \ref{Appendix:PLAT} until $\residual = 0$. With the same $\mu>0$ we have exponential convergence until $\Phi/2$. As we said in Appendix \ref{Appendix:PLAT} as $\residual$ crosses $\Phi$, the norm $\scale$ restarts increasing. This implies that a good lower bound remains $Q$ of the time of crossing. The evolution of $\residual$
\begin{equation}
    \residual(k+1)
    \ = \ (1 - \eta \scale) \residual(k)  + \eta^2 \residual(k)^2 (\residual(k) + \Phi)
    \ \geq \ (1 - \eta Q) \residual(k).
\end{equation}
The time $t$ taken to $\residual$ to go from $\Phi$ to $\Phi/2$ is thus
\begin{equation}
    \Phi/2 \ \geq \ (1 - \eta Q)^t \Phi
\end{equation}
so we have
\begin{equation}
    t \leq \frac{\log(\Phi) - \log(\Phi/2)}{-\log(1 - \eta Q)} \leq 
    \frac{\log(\Phi) - \log(\Phi/2)}{\eta Q_\tau}.
\end{equation}

\subsection{Closing up: Tight rate}
The previous sections and Lemma \ref{lemma:conv:loss} allow us to conclude that we have loss convergence, i.e., $L\leq \delta$, in a number of steps which is
\begin{equation}
    t \ \ \leq \ \
    \tau \ + \
    \frac{\log(\Phi) - \log(\Phi/2)}{\eta \min\{Q_\tau, 2\Phi\}} 
    \ + \ 2 \ + \
    \frac{\log(\Phi/2) - \log(\delta)}{\eta \min\{Q_\tau, \Phi\}},
\end{equation}
where $\tau$ is the $\tau_1$ defined in Definition \ref{def:sequence} and evaluated in Proposition \ref{prop:final:PLAT}.
This establishes Theorem \ref{theo:speed}.

\section{Curiosity: Jumps between regions}
Note that if the dynamics does not jump from one side to the other of the landscape, then we have a clean exponential convergence and we can control the implicit regularization. We will see under which hypothesis on the learning rate this happens.

Note that Equation \ref{eq:residuals-update} tells us that after one step $\residual$ does not change sign (thus you remain in the same region in which you started) if and only if we have the following bound on the learning rate.
\begin{definition}
    For all $\a, \b \in \R^n$, let $\alpha = \frac{\residual ( \residual + \Phi )}{\scale^2}$, define
    \begin{equation}
        \eta_1
        \quad := \quad 
        \frac{1}{\scale} \left( 1 + \alpha + 2\alpha^2 + 5\alpha^3 + 14 \alpha^8 + \ldots \right),
    \end{equation}
    \begin{equation}
        \eta_2
        \quad := \quad 
        \frac{2}{\scale} \left( 1 + 2\alpha + 8\alpha^2 + 40\alpha^3 + 224\alpha^4 + \ldots \right).
    \end{equation}
\end{definition}

The way we obtain $\eta_1$ is by seeing for what $\eta$ we have that $\residual(t+1) = 0$. Precisely, 
\begin{lemma}
\label{lemma:eta}
    If $\eta = \eta_1$, we have that the residuals at the next steps are 0. If $\eta = \eta_2$, then the residuals at the next steps are the same but changed of sign. Moreover,
    \begin{itemize}[leftmargin=1em,itemsep=0.2em,parsep=0.1em]
        \item If $\eta \in (0, \eta_1)$ then $sign(\residual(1)) = sign(\residual)$ and $|\residual(1)| < |\residual|$.
        \item If $\eta \in (\eta_1, \eta_2)$ then $sign(\residual(1)) \neq sign(\residual)$ and $|\residual(1)| < |\residual|$.
    \end{itemize}
\end{lemma}
\begin{proof}[Proof of Lemma \ref{lemma:eta}.]
Note that the residuals after one step are the same sign as the previous residuals if and only if 
\begin{equation}
    1 - \eta \scale + \eta^2 \residual ( \residual + \Phi ) \quad \geq \quad 0.
\end{equation}
Solving this one as a second degree equation gives
\begin{equation}
    \eta 
    \quad \leq \quad
    \frac{\scale-\sqrt{\scale^2 - 4\residual ( \residual + \Phi )}}{2\residual ( \residual + \Phi )}
    \quad \text{or} \quad
    \eta 
    \quad \geq \quad
    \frac{\scale+\sqrt{\scale^2 - 4\residual ( \residual + \Phi )}}{2\residual ( \residual + \Phi )}
\end{equation}
Now expanding in Taylor the square root, we obtain that 
\begin{equation}
    \eta_1 
    \quad \leq \quad
    \frac{1}{2\residual ( \residual + \Phi )}
    \left( 
        \frac{4\residual ( \residual + \Phi )}{2\scale}
        + 
        \frac{16\residual^2 ( \residual + \Phi )^2}{8\scale^3}
        +
        \ldots
    \right)
    \quad = \quad
    \frac{1}{\scale} + \frac{\residual ( \residual + \Phi )}{\scale^3}
    + \ldots
\end{equation}
This implies that the residuals are the same sign as the starting ones if
\begin{equation}
    \eta \quad \leq \quad
        \eta_1
    \quad \text{or} \quad
    \eta 
    \quad \geq \quad
    \frac{\scale}{\residual ( \residual + \Phi )} - \eta_1.
\end{equation}
Analogously, for $\eta_2$ we have that the absolute value of the residuals is smaller than the absolute value of the residuals one step before, if and only if
\begin{equation}
    2 - \eta \scale + \eta^2 \residual ( \residual + \Phi ) \quad \geq \quad 0.
\end{equation}
This implies that 
\begin{equation}
    \eta 
    \quad \leq \quad
    \frac{\scale-\sqrt{\scale^2 - 8\residual ( \residual + \Phi )}}{2\residual ( \residual + \Phi )}
    \quad \text{or} \quad
    \eta 
    \quad \geq \quad
    \frac{\scale+\sqrt{\scale^2 - 8\residual ( \residual + \Phi )}}{2\residual ( \residual + \Phi )}
\end{equation}
and analogously to before
\begin{equation}
    \eta \quad \leq \quad
        \eta_2
    \quad \text{or} \quad
    \eta 
    \quad \geq \quad
    \frac{\scale}{\residual ( \residual + \Phi )} - \eta_2.
\end{equation}
Also note that for $\residual > 0$ we have that $\residual(1)<\residual$ or for $\residual < 0$ we have that $\residual(1)>\residual$ if and only if
\begin{equation}
    1 - \eta \scale + \eta^2 \residual ( \residual + \Phi ) \quad \leq \quad 1.
\end{equation}
This solves when 
\begin{equation}
    \eta \quad \leq \quad \frac{\scale}{ \residual ( \residual + \Phi ) }
\end{equation}
% With a different techniques, by updating iteratively $\eta$ to delete the higher order terms we find that $\eta_1$ is the limit of the following sequence
% \begin{equaton}
%     \frac{1}{\scale}, \quad \frac{1}{\scale} + \frac{\residual( \residual + \Phi )}{\scale^3}, \quad 
%     .
% \end{equaton}
% Indeed, 
% \begin{equation}
%     1 - \frac{1}{\scale} \scale + \frac{1}{\scale^2}\residual ( \residual + \Phi ) \quad = \quad \frac{\residual( \residual + \Phi )}{\scale^2},
% \end{equation}
% \begin{equation}
%     1 - \frac{1}{\scale} - \frac{\residual( \residual + \Phi )}{\scale^3} \scale + \frac{1}{\scale^2}\residual ( \residual + \Phi ) +
    
%     \quad = \quad \frac{\residual( \residual + \Phi )}{\scale^2},
% \end{equation}
\end{proof}

Note that what we did here implies that if $\eta \leq \eta_2$ and $\eta \leq \frac{\sqrt{2}}{\residual}$ for all the $\scale$s along the trajectory we thus always have exponential convergence if such PL condition holds. We know from the previous section that in this setting $\scale$ is always smaller than $\bar \scale$. So if such a $\mu$ exists and $\eta \leq \eta_2$ with $\bar \scale$ and we converge and we can properly bound the implicit regularization.

% More precisely, we need to have everywhere along the trajectory
% \begin{equation}
%     \eta \leq \min_{\scale, \residual} \frac{2}{\scale} \left( 1 + \frac{2\residual (\residual + \Phi)}{\scale^2} + \ldots \right)
% \end{equation}
% minimizing over the values of $\a^\top \b$ the enumerator and maximizing over $\a^\top \b$ the denominator that could be taken, again here we have that this is assured when
% \begin{equation}
%     \eta \leq \frac{2}{\bar \scale} \left( 1 - \frac{\Phi^2}{2\bar \scale^2} + \frac{\Phi^4}{\bar \scale^4} + \ldots \right).
% \end{equation}
% Note that this is a very tight lower bound and it applies only when we are initializing in the area where $0\leq \a ^\top \b \leq \Phi$. Precisely, the minimimum is for $\a ^\top \b = \Phi/2$.

\section{Location of Convergence - Proof of Theorem \ref{theo:loc}}
\label{app:location}
We will bound here the final $Q$ for two reasons:
\begin{itemize}[leftmargin=1em,itemsep=0.2em,parsep=0.1em]
    \item Understanding the location of convergence.
    \item Picking the right learning rate.
\end{itemize}
Note that assuming $\eta \leq \frac{\sqrt{2}}{\residual}$ along the whole trajectory
% , implies that
% \begin{equation}
%     \residual(1)
%     \quad  \leq \quad 
%     \residual - \sqrt{2}\scale + 2 \big( \residual(t) + \Phi \big).
% \end{equation}
we have that $Q$ strictly monotonically shrinks along the trajectory.
This means that the dynamics may seem to oscillate around in an uncontrollable way, but every time it oscillates is landing on a trajectory that takes to a global minimum with lower $Q$.

Note that this is true almost everywhere, indeed if the trajectory is such that at a certain point in time $t$ satisfies $\eta = \residual(t)^{-1}$ exactly, then the trajectory would land on the trajectory taking to the saddle, indeed
\begin{equation}
    \a(t+1) = -\b(t+1) = \a(t)-\b(t).
\end{equation}
Luckily, fixing a learning rate size, the set of starting points for which this is the case has measure zero. Observe also that $\eta = -\residual^{-1}$ is instead optimal and results in $\a(t+1) = \b(t+1)$, implying convergence to a balanced solution.
This means that assuming $\eta \leq \frac{\sqrt{2}}{\residual}$ implies that the dynamics may diverge or converge, but for sure at every step is getting closer and closer to the subspace in which $\a = \b$.
Moreover, note that all the $Q_i$ change sign if and only if $\eta |\residual| > 1$.

% \red{Note that picking any learning rate $\eta$ such that (i) $|\residual(t+1)| \leq |\residual(t)|$ for all $t$ and $\eta \leq \frac{\sqrt{2}}{\residual}$ is thus sufficient for the convergence. It is not necessary though, it may be that the absolute value of the residuals grows and then descents.}

Regarding the proof of the upperbound of Theorem \ref{theo:loc} note that for all $t$
\begin{equation}
\begin{split}
    Q_i(t) 
    \quad &= \quad 
    Q_i(0) \cdot  \prod_{k=0}^{t-1} (1 - \eta^2 \residual_k^2)
    \quad = \quad
    Q_i(0) \cdot \exp \left( \sum_{k=0}^{t-1} \log(1 - \eta^2 \residual_k^2)\right)
    .
\end{split}
\end{equation}
In absolute value, we can thus upperbound the RHS as follows, by applying the Taylor expansion whenever $\eta |\residual| < 1$
\begin{lemma}[Upperbound to the inbalance, 1]
Let $\eta |\residual(t)| < 1$ for all $t \in \N$, then for all $t \in \N$ we have
\[
|Q_i(t)|
\ < \
|Q_i(0)| \cdot \exp \left( -\eta^2 \sum_{k=0}^{t-1} \residual_k^2\right).
\]
\end{lemma}

By combining this lemma and Lemma \ref{lemma:conserved_norm} we obtain
\begin{lemma}[Upperbound to the inbalance, 2]
Let $\eta |\residual(0)| < 1$ and $\eta \leq \tilde \eta$, then for all $t \in \N$ 
\[
|Q_i(t|
\ < \
|Q_i(0)| \cdot \exp \left( -\eta^2 \sum_{k=0}^{t-1} \residual_k^2\right).
\]
\end{lemma}
This establishes the upper bound of Theorem \ref{theo:loc}.
Regarding the proof of the lower bound, notice that we have from Appendix \ref{appendix:mu_cases} that the rate of convergence of $\residual$ is at least $Q(\tau_1)$ in Region B and at least $2\Phi$ in Region A, once adding the right assumption on the learning rate. This implies that if the initialization is in Region B or C, then
\begin{lemma}[Lower bound to the imbalance]
Assume there exists $\tilde t$ such that for all $t \geq \tilde t$ we have $\eta |\residual(0)| < 1/2$ then
\[
Q_i(t) 
\]
\end{lemma}
% \red{invert them.}
\begin{proof}
Note that the fact that $\eta |\residual(k)| < 1/2$ for all $k$ makes sure that
\begin{equation}
\begin{split}
    Q_i(t) \ &= \ 
    Q_i(0) \cdot  \prod_{k=0}^{t-1} (1 - \eta^2 \residual_{k}^2)
    \ \geq \
    Q_i(0) \cdot  \exp \left( - \sum_{k=0}^{t-1} \eta^2 \residual_{k}^2 - \eta^4 \residual_{k}^4
    \right)
\end{split}
\end{equation}
since for $x \in [0,1/2]$ we have $1 - x > e^{-x-x^2}$.
Next note that in these hypothesis, by Theorem \ref{theo:speed}, we ahve exponential convergence, thus
\begin{equation}
\begin{split}
    Q_i(t) \ & \geq \
    Q(0) \cdot  \exp \left( - \eta^2 \residual(0)^2 \sum_{k=1}^\infty (1 - \eta Q(\tau))^{2k} - \eta^4 \residual_{0}^4 \sum_{1}^\infty (1 - \eta Q(\tau) )^{4k} \right)
    \\&\geq \
    Q(0) \cdot  \exp \left( - \frac{\eta \residual(0)^2}{Q(\tau) (2 - \eta Q(\tau))}  - \frac{\eta^3 \residual(0)^4}{Q(\tau) (8 - \eta Q(\tau))}  \right)
    \\&\geq \
    Q(0) \cdot  \exp \left( - \frac{\sqrt{\eta} \residual(0)^2}{2\Phi}\left(1 + \frac{\eta^2\residual(0)^2}{8} \right)  \right)
    \\&\geq \
    Q(0) \cdot  \exp \left( - \frac{\sqrt{\eta} \residual(0)^2}{\Phi} \right).
    % Q(\tau_1) \exp \left( -\eta^2\Phi^2 \sum_{k=\tau_1}^\infty \log(1 - (1 - \eta Q (\tau_1))^{2k})\right)
    % \\ &\geq \
    % Q(\tau_1) \exp \left( - \sum_{k=\tau_1}^\infty \log(1 - (1 - \eta Q (\tau_1))^{2k})\right)
    % % \\& < \
    % % Q(\tau_1) \exp \left( -\frac{\eta \Phi^2}{2 Q (\tau_1)} \right)
\end{split}
\end{equation}
By plugging in the lower bound in Lemma \ref{lemma:conv:1}. 
This concludes the proof of the lemma. 
\end{proof}
This concludes the proof of Theorem \ref{theo:loc}. 

% and
% \begin{equation}
% \begin{split}
%     Q_\infty \ &= \ 
%     Q(\tau_1) \cdot  \prod_{k=\tau_1}^\infty (1 - \eta^2 \residual_{k}^2)
%     \ \leq \
%     Q(\tau_1) \cdot  \exp \left( - \eta^2 \residual_{\tau_1}^2 \sum_{k=\tau_1} \log(1 - 2\eta \Phi)^{2k} \right)
%     \\ &\geq \
%     Q(\tau_1) \cdot  \exp \left( - \frac{\eta \residual_{\tau_1}^2}{4\Phi (1- \eta \Phi)}  - \frac{\eta^3 \residual_{\tau_1}^4}{8\Phi (1 - O(\eta)} \right).
% \end{split}
% \end{equation}
% In the case of Region $A$ instead
% \begin{equation}
% \begin{split}
%     Q_\infty \ &\leq \ 
%     \ < \
%     Q_0 \exp \left( -\eta^2\residual(0)^2 \sum_{k=\tau_1}^\infty (1 - 2\eta \Phi)^2\right)
%     \\& < \
%     Q(\tau_1) \exp \left( -\frac{\eta \residual(0)^2}{4 \Phi} \right)
% \end{split}
% \end{equation}
% and
% \begin{equation}
% \begin{split}
%     Q_\infty \ &= \ 
%     Q(\tau_1) \cdot  \prod_{k=\tau_1}^\infty (1 - \eta^2 \residual_{k}^2)
%     \ \geq \
%     Q(\tau_1) \cdot  \exp \left( - \eta^2 \residual_{0}^2 \sum_{k=0} (1 - 2\eta \scale)^{2k} - \eta^4 \residual_{0}^4 \sum_{k=\tau_1} (1 - 2\eta \Phi)^{4k} \right)
%     \\ &\geq \
%     Q(\tau_1) \cdot  \exp \left( - \frac{\eta \residual_{\tau_1}^2}{4\Phi (1- \eta \Phi)}  - \frac{\eta^3 \residual_{\tau_1}^4}{8\Phi (1 - O(\eta)} \right).
% \end{split}
% \end{equation}

\end{document}

%% file: pier_macros.tex
\newcommand{\qandq}{\quad \text{and} \quad}

\newcommand{\mednorm}[1]{\| #1\| }

%% file: tikz_stuff.tex
\usepackage{pgfplots}
\pgfplotsset{compat=1.18}
\usepgfplotslibrary{fillbetween}
\usetikzlibrary{shapes,arrows,positioning,decorations.markings}

\newcommand{\maketikzabplot}[0]{
\begin{tikzpicture}
    \begin{axis}[
        domain=-3:3,
        samples=200,
        axis lines=middle,
        axis line style={gray},
        xlabel={$a$},
        ylabel={$b$},
        xmin=-3.5, xmax=3.5,
        ymin=-3.5, ymax=3.5,
        xtick=\empty,
        ytick=\empty,
        restrict y to domain=-3:3,
        ]
        \addplot [
            domain=0:3,
            samples=200,
            name path=upper1,
            draw=none,
        ] {min(1/x,3)};
        \addplot [
            domain=0:3,
            samples=200,
            name path=lower1,
            draw=none,
        ] {0};
        \addplot [
            fill=blue,
            fill opacity=0.15
        ] fill between [
            of=lower1 and upper1,
        ];
        \addplot [
            domain=0.334:3,
            samples=200,
            name path=upper3,
            draw=none,
        ] {3};
        \addplot [
            domain=0.334:3,
            samples=200,
            name path=lower3,
            draw=none,
        ] {1/x};
        \addplot [
            fill=green,
            fill opacity=0.15
        ] fill between [
            of=lower3 and upper3,
        ];
        \addplot [
            domain=0:3,
            samples=200,
            name path=upper5,
            draw=none,
        ] {0};
        \addplot [
            domain=0:3,
            samples=200,
            name path=lower5,
            draw=none,
        ] {-3};
        \addplot [
            fill=orange,
            fill opacity=0.15
        ] fill between [
            of=lower5 and upper5,
        ];
        \addplot [
            domain=-3:0,
            samples=200,
            name path=lower2,
            draw=none,
        ] {max(1/x,-3)};
        \addplot [
            domain=-3:0,
            samples=200,
            name path=upper2,
            draw=none,
        ] {0};
        \addplot [
            fill=blue,
            fill opacity=0.15
        ] fill between [
            of=lower2 and upper2,
        ];
        \addplot [
            domain=-3:0.334,
            samples=200,
            name path=lower4,
            draw=none,
        ] {-3};
        \addplot [
            domain=-3:0.334,
            samples=200,
            name path=upper4,
            draw=none,
        ] {1/x};
        \addplot [
            fill=green,
            fill opacity=0.15
        ] fill between [
            of=lower4 and upper4,
        ];
        \addplot [
            domain=-3:0,
            samples=200,
            name path=lower6,
            draw=none,
        ] {0};
        \addplot [
            domain=-3:0,
            samples=200,
            name path=upper6,
            draw=none,
        ] {3};
        \addplot [
            fill=orange,
            fill opacity=0.15
        ] fill between [
            of=lower6 and upper6,
        ];
        \addplot[
            domain=0.33:3,
            thick,
            black
        ] {1/x};
        \addplot[
            domain=-3:-0.33,
            thick,
            black
        ] {1/x};
        \addplot[
            domain=-3:3,
            thick,
            dashed,
            black
        ] {-x};
        \node at (axis cs:0.5,0.5) {B};
        \node at (axis cs:-0.5,-0.5) {B};
        \node at (axis cs:1.5,1.5) {A};
        \node at (axis cs:-1.5,-1.5) {A};
        \node at (axis cs:-1.5,0.5) {C};
        \node at (axis cs:-0.5,1.5) {C};
        \node at (axis cs:1.5,-0.5) {C};
        \node at (axis cs:0.5,-1.5) {C};
    \end{axis}
\end{tikzpicture}
}

\newcommand{\gdupdatepicture}[0]{
\begin{tikzpicture}
    \begin{axis}[
        domain=-0.5:3,
        samples=200,
        axis lines=middle,
        axis line style={gray},
        xlabel={$a$},
        ylabel={$b$},
        xmin=-0.25, xmax=3,
        ymin=-2, ymax=2,
        xtick=\empty,
        ytick=\empty,
        restrict y to domain=-2:2,
        ]
        \addplot[
            domain=0.33:3,
            thick,
            black
        ] {1/x};
        \addplot[
            domain=1:1.272,
            thick,
            red,
            postaction={decorate, decoration={markings, mark=at position 1 with {\arrow[scale=1.5]{>}}}}
        ] {sqrt(x^2 - 1)};
        \addplot[
            domain=1:2,
            thick,
            red
        ] {-sqrt(x^2 - 1)};
        \addplot[
            only marks,
            mark=*,
            mark size=1.5,
            color=black
        ] coordinates {(2,-1.732)};
        \addplot[
            only marks,
            mark=*,
            mark size=1.5,
            color=blue
        ] coordinates {(1.03, 0.247)};
        \addplot[
            only marks,
            mark=*,
            mark size=1.5,
            color=black
        ] coordinates {(1,1)};
        \draw[->, line width=0.4mm,color=blue] (1.03, 0.247) -- (1.14, 0.708);
        \node[color=red] at (axis cs:1.9,-0.8) {GF Trajectory};
        \node[color=blue] at (axis cs:0.74, 0.49) {GD step};
        \node at (axis cs:1.4, 1.2) {$a=b=\sqrt{\Phi}$};
    \end{axis}  
\end{tikzpicture}
}

%% file: blake_macros.tex
\DeclarePairedDelimiter{\abs}{\lvert}{\rvert} %
\DeclarePairedDelimiter{\brk}{[}{]}
\DeclarePairedDelimiter{\crl}{\{}{\}}
\DeclarePairedDelimiter{\prn}{(}{)}
\DeclarePairedDelimiter{\nrm}{\|}{\|}

\newcommand{\inner}[2]{\left\langle #1,\, #2 \right\rangle}

\DeclareMathOperator{\relu}{\mathrm{ReLU}}

\DeclareMathOperator*{\argmin}{arg\,min}

\newcommand{\collapse}[1]{$\dots$}

\newcommand{\R}{{\mathbb{R}}}
\newcommand{\N}{{\mathbb{N}}}

\newcommand{\scale}[0]{\boldsymbol{\lambda}}
\newcommand{\residual}[0]{\boldsymbol{\varepsilon}}

\renewcommand{\a}[0]{\mathbf{a}}
\renewcommand{\b}[0]{\mathbf{b}}
\newcommand{\atb}[0]{\a^\top \b}

\newcommand{\x}[0]{\mathbf{x}}
\newcommand{\y}[0]{\mathbf{y}}

%% file: ICML_main.bbl
\begin{thebibliography}{22}
\providecommand{\natexlab}[1]{#1}
\providecommand{\url}[1]{\texttt{#1}}
\expandafter\ifx\csname urlstyle\endcsname\relax
  \providecommand{\doi}[1]{doi: #1}\else
  \providecommand{\doi}{doi: \begingroup \urlstyle{rm}\Url}\fi

\bibitem[Ahn et~al.(2024)Ahn, Bubeck, Chewi, Lee, Suarez, and Zhang]{ahn_learning_2024}
Ahn, K., Bubeck, S., Chewi, S., Lee, Y.~T., Suarez, F., and Zhang, Y.
\newblock Learning threshold neurons via edge of stability.
\newblock In \emph{Advances in {Neural} {Information} {Processing} {Systems}}, volume~36, 2024.

\bibitem[Arora et~al.(2019)Arora, Cohen, Golowich, and Hu]{arora_convergence_2019}
Arora, S., Cohen, N., Golowich, N., and Hu, W.
\newblock A {Convergence} {Analysis} of {Gradient} {Descent} for {Deep} {Linear} {Neural} {Networks}, October 2019.
\newblock URL \url{http://arxiv.org/abs/1810.02281}.
\newblock arXiv:1810.02281 [cs, stat].

\bibitem[Bjorck et~al.(2018)Bjorck, Gomes, Selman, and Weinberger]{bjorck_understanding_2018}
Bjorck, N., Gomes, C.~P., Selman, B., and Weinberger, K.~Q.
\newblock Understanding batch normalization.
\newblock In \emph{Advances in neural information processing systems}, volume~31, 2018.

\bibitem[Bottou et~al.(2018)Bottou, Curtis, and Nocedal]{bottou_optimization_2018}
Bottou, L., Curtis, F.~E., and Nocedal, J.
\newblock Optimization {Methods} for {Large}-{Scale} {Machine} {Learning}, February 2018.
\newblock URL \url{http://arxiv.org/abs/1606.04838}.
\newblock arXiv: 1606.04838.

\bibitem[Chen \& Bruna(2023)Chen and Bruna]{chen_beyond_2023}
Chen, L. and Bruna, J.
\newblock Beyond the edge of stability via two-step gradient updates.
\newblock In \emph{International {Conference} on {Machine} {Learning}}, pp.\  4330--4391. PMLR, 2023.

\bibitem[Cohen et~al.(2021)Cohen, Kaur, Li, Kolter, and Talwalkar]{cohen_gradient_2021}
Cohen, J., Kaur, S., Li, Y., Kolter, J.~Z., and Talwalkar, A.
\newblock Gradient descent on neural networks typically occurs at the edge of stability.
\newblock In \emph{International {Conference} on {Learning} {Representations}}, 2021.

\bibitem[Gidel et~al.(2019)Gidel, Bach, and Lacoste-Julien]{gidel_implicit_2019}
Gidel, G., Bach, F., and Lacoste-Julien, S.
\newblock Implicit {Regularization} of {Discrete} {Gradient} {Dynamics} in {Linear} {Neural} {Networks}.
\newblock In \emph{Advances in {Neural} {Information} {Processing} {Systems}}, volume~32. Curran Associates, Inc., 2019.
\newblock URL \url{https://proceedings.neurips.cc/paper_files/paper/2019/hash/f39ae9ff3a81f499230c4126e01f421b-Abstract.html}.

\bibitem[Hochreiter \& Schmidhuber(1997)Hochreiter and Schmidhuber]{hochreiter_flat_1997}
Hochreiter, S. and Schmidhuber, J.
\newblock Flat minima.
\newblock \emph{Neural Computation}, 9\penalty0 (1):\penalty0 1--42, 1997.
\newblock Publisher: MIT Press.

\bibitem[Jastrzebski et~al.(2020)Jastrzebski, Szymczak, Fort, Arpit, Tabor, Cho, and Geras]{jastrzebski_break-even_2020}
Jastrzebski, S., Szymczak, M., Fort, S., Arpit, D., Tabor, J., Cho, K., and Geras, K.
\newblock The {Break}-{Even} {Point} on {Optimization} {Trajectories} of {Deep} {Neural} {Networks}, February 2020.
\newblock URL \url{http://arxiv.org/abs/2002.09572}.
\newblock arXiv: 2002.09572.

\bibitem[Keskar et~al.(2016)Keskar, Mudigere, Nocedal, Smelyanskiy, and Tang]{keskar_large-batch_2016}
Keskar, N.~S., Mudigere, D., Nocedal, J., Smelyanskiy, M., and Tang, P. T.~P.
\newblock On large-batch training for deep learning: {Generalization} gap and sharp minima, 2016.

\bibitem[LeCun et~al.(2002)LeCun, Bottou, Orr, and Müller]{lecun_efficient_2002}
LeCun, Y., Bottou, L., Orr, G.~B., and Müller, K.-R.
\newblock Efficient backprop.
\newblock In \emph{Neural networks: {Tricks} of the trade}, pp.\  9--50. Springer, 2002.

\bibitem[Lewkowycz et~al.(2020)Lewkowycz, Bahri, Dyer, Sohl-Dickstein, and Gur-Ari]{lewkowycz_large_2020}
Lewkowycz, A., Bahri, Y., Dyer, E., Sohl-Dickstein, J., and Gur-Ari, G.
\newblock The large learning rate phase of deep learning: the catapult mechanism, March 2020.
\newblock URL \url{http://arxiv.org/abs/2003.02218}.
\newblock arXiv:2003.02218 [cs, stat].

\bibitem[Li et~al.(2019)Li, Wei, and Ma]{li_towards_2019}
Li, Y., Wei, C., and Ma, T.
\newblock Towards explaining the regularization effect of initial large learning rate in training neural networks.
\newblock In \emph{Advances in neural information processing systems}, volume~32, 2019.

\bibitem[Nguegnang et~al.(2024)Nguegnang, Rauhut, and Terstiege]{nguegnang_convergence_2024}
Nguegnang, G.~M., Rauhut, H., and Terstiege, U.
\newblock Convergence of gradient descent for learning linear neural networks.
\newblock \emph{Advances in Continuous and Discrete Models}, 2024\penalty0 (1):\penalty0 1--28, 2024.
\newblock Publisher: Springer.

\bibitem[Park et~al.(2019)Park, Sohl-Dickstein, Le, and Smith]{park_effect_2019}
Park, D., Sohl-Dickstein, J., Le, Q., and Smith, S.
\newblock The {Effect} of {Network} {Width} on {Stochastic} {Gradient} {Descent} and {Generalization}: an {Empirical} {Study}.
\newblock In \emph{Proceedings of the 36th {International} {Conference} on {Machine} {Learning}}, pp.\  5042--5051. PMLR, May 2019.
\newblock URL \url{https://proceedings.mlr.press/v97/park19b.html}.
\newblock ISSN: 2640-3498.

\bibitem[Polyak(1963)]{polyak_gradient_1963}
Polyak, B.~T.
\newblock Gradient methods for the minimisation of functionals.
\newblock \emph{USSR Computational Mathematics and Mathematical Physics}, 3\penalty0 (4):\penalty0 864--878, 1963.
\newblock Publisher: Elsevier.

\bibitem[Saxe et~al.(2014)Saxe, McClelland, and Ganguli]{saxe_exact_2014}
Saxe, A.~M., McClelland, J.~L., and Ganguli, S.
\newblock Exact solutions to the nonlinear dynamics of learning in deep linear neural networks, February 2014.
\newblock URL \url{http://arxiv.org/abs/1312.6120}.
\newblock arXiv:1312.6120 [cond-mat, q-bio, stat].

\bibitem[Smith \& Le(2018)Smith and Le]{smith_bayesian_2018}
Smith, S.~L. and Le, Q.~V.
\newblock A {Bayesian} {Perspective} on {Generalization} and {Stochastic} {Gradient} {Descent}, February 2018.
\newblock URL \url{http://arxiv.org/abs/1710.06451}.
\newblock arXiv: 1710.06451.

\bibitem[Tarmoun et~al.(2021)Tarmoun, Franca, Haeffele, and Vidal]{tarmoun_understanding_2021}
Tarmoun, S., Franca, G., Haeffele, B.~D., and Vidal, R.
\newblock Understanding the dynamics of gradient flow in overparameterized linear models.
\newblock In \emph{International {Conference} on {Machine} {Learning}}, pp.\  10153--10161. PMLR, 2021.

\bibitem[Wang et~al.(2022)Wang, Chen, Zhao, and Tao]{wang_large_2022}
Wang, Y., Chen, M., Zhao, T., and Tao, M.
\newblock Large {Learning} {Rate} {Tames} {Homogeneity}: {Convergence} and {Balancing} {Effect}.
\newblock In \emph{International {Conference} on {Learning} {Representations}}, 2022.
\newblock URL \url{https://openreview.net/forum?id=3tbDrs77LJ5}.

\bibitem[Xu \& Ziyin(2024)Xu and Ziyin]{xu_three_2024}
Xu, Y. and Ziyin, L.
\newblock Three {Mechanisms} of {Feature} {Learning} in the {Exact} {Solution} of a {Latent} {Variable} {Model}, May 2024.
\newblock URL \url{http://arxiv.org/abs/2401.07085}.
\newblock arXiv:2401.07085.

\bibitem[Xu et~al.(2023)Xu, Min, Tarmoun, Mallada, and Vidal]{xu_linear_2023}
Xu, Z., Min, H., Tarmoun, S., Mallada, E., and Vidal, R.
\newblock Linear convergence of gradient descent for finite width over-parametrized linear networks with general initialization.
\newblock In \emph{International {Conference} on {Artificial} {Intelligence} and {Statistics}}, pp.\  2262--2284. PMLR, 2023.

\end{thebibliography}
